\documentclass[10pt,a4paper,reqno]{amsart}

\usepackage[utf8]{inputenc}
\usepackage{fourier}
\usepackage[scaled]{berasans}
\usepackage[scaled]{beramono}
\usepackage{amssymb}
\usepackage{MnSymbol}
\usepackage{tikz}
  \usetikzlibrary{arrows,positioning}
\usepackage{microtype}
\usepackage{graphicx}
\usepackage{mathtools,amsthm}
\usepackage{enumitem}
\usepackage{natbib}
\usepackage{hyperref}
\usepackage[hypcap]{caption}
\usepackage{nicefrac}
\allowdisplaybreaks
\usepackage{xargs}

\usepackage{comment}

\newcommand*{\rval}{x}
\newcommand*{\ps}[1][\rv]{\mathcal{#1}}
\newcommand*{\cf}[1][]{C_{#1}}
\newcommand*{\rf}[1][]{R_{#1}}
\newcommand*{\vs}{\mathcal{V}}

\newcommand*{\gambles}{\mathcal{L}}
\newcommand*{\gambleson}[1][\ps]{\mathcal{L}(#1)}
\newcommand*{\nonneggambles}{\gambles_{\geq0}}

\newcommand*{\posgambles}{\gambles_{>0}}

\newcommand*{\cfdom}{\mathcal{Q}}

\newcommand*{\cfdomo}{\mathcal{Q}_0}
\newcommand*{\reals}{\mathbb{R}}
\newcommand*{\posreals}{\reals_{>0}}
\newcommand*{\negreals}{\reals_{<0}}
\newcommand*{\nonnegreals}{\reals_{\geq0}}

\newcommand*{\nats}{\mathbb{N}}

\newcommand*{\cset}[3][]{\set[#1]{#2:#3}}
\newcommand*{\os}[1][]{A_{#1}}

\newcommand*{\rel}{\lhd}
\newcommand*{\crel}[1][]{\rel_{#1}}

\renewcommand*{\iff}{\Leftrightarrow}
\newcommand*{\then}{\Rightarrow}
\newcommand*{\vect}[1][u]{#1}
\newcommand*{\cone}{\mathcal{K}}
\newcommand*{\slo}[1][\mf]{\prec_{#1}}

\newcommand*{\nslo}[1][\mf]{{\not\prec}_{#1}}
\newcommand*{\rslo}[1][\mf]{\preceq_{#1}}
\newcommand*{\nrslo}[1][\mf]{\not\preceq_{#1}}

\newcommand*{\ilo}[1][\mf]{\parallel_{#1}}
\newcommand*{\vo}[1][]{\preceq_{#1}}
\newcommand*{\voi}[1][]{\succeq_{#1}}
\newcommand*{\svo}[1][]{\prec_{#1}}

\newcommand*{\svoi}[1][]{\succ_{#1}}
\newcommand*{\nsvoi}[1][]{\mathbin{{\nsucc}_{#1}}}
\newcommand*{\vspos}[1][\svoi]{\vs_{{#1}0}}

\newcommand*{\nmi}[1][]{\sqsubseteq_{#1}}

\newcommand*{\poset}[2]{\group{{#1};{#2}}}

\newcommand*{\allcfs}{\mathcal{C}}

\newcommand*{\allcohcfs}{\bar{\allcfs}}

\newcommand*{\cfs}[1][']{\allcfs{#1}}

\newcommand*{\hl}[1][]{H_{#1}}
\newcommand*{\hls}[1][]{\mathcal{H}_{#1}}
\newcommand*{\rew}{r}
\newcommand*{\rewtoo}{s}
\newcommand*{\rewards}[1][]{\mathcal{R}_{#1}}
\newcommand*{\bestr}{\top}
\newcommand*{\worstr}{\bot}
\newcommand*{\rewardsw}[1][\worstr]{\rewards[#1]}

\newcommand*{\poh}[1][\worstr]{\varphi_{#1}}
\newcommand*{\lift}[1][\worstr]{\tilde{\varphi}_{#1}}

\newcommand*{\RN}[1]{%
  \textup{\uppercase\expandafter{\romannumeral#1}}%
}

\newcommand*{\lexi}{\mathrm{L}}

\newcommand*{\sodv}[1][]{D_{#1}}

\newcommand*{\prel}[1][]{\closedprec_{#1}}
\newcommand*{\nprel}[1][]{\not\closedprec_{#1}}
\newcommand*{\allsodvs}{\mathcal{D}}

\newcommand*{\allcohsodvs}{\bar{\allsodvs}}
\newcommand*{\maxcohsodvs}{\hat{\allsodvs}}
\newcommand*{\maxcohsodvsdom}[1]{{\hat{\allsodvs}}_{#1}}

\newcommand*{\mf}[1][]{p_{#1}}
\newcommand*{\E}[1][\mf]{\mathrm{E}_{#1}}

\newcommand*{\ev}{B}

\newcommand*{\ind}[1][\ev]{\mathbb{I}_{#1}}
\newcommand*{\indset}[1]{\ind[\set{#1}]}

\newcommand*{\lo}[1][]{<_\mathrm{L}^{#1}}
\newcommand*{\rlo}[1][]{\leq_\mathrm{L}^{#1}}

\newcommand*{\cohlexisodvs}{\allcohsodvs_\mathrm{L}}
\newcommand*{\eqrel}[1][k]{\simeq_j}
\newcommandx{\ec}[2][1=f, 2=k]{{#1}/\ker\Lambda_{#2}}
\newcommandx{\qs}[2][1=\gambleson, 2=k]{{#1}/\ker\Lambda_{#2}}

\newcommand*{\lp}[1][]{\underline{P}_{#1}}

\newcommand*{\alllexicohcfs}{\allcohcfs_\mathrm{L}}
\newcommand*{\alllexicohsodvs}{\allcohsodvs_\lexi}

\DeclareMathOperator*{\Posi}{posi}
\DeclareMathOperator*{\Span}{span}
\DeclareMathOperator*{\ch}{CH}

\DeclareMathOperator*{\interior}{int}
\DeclareMathOperator*{\closure}{cl}
\DeclareMathOperator*{\relinterior}{ri}
\DeclareMathOperator*{\domain}{dom}

\DeclarePairedDelimiter{\group}{(}{)}

\DeclarePairedDelimiter{\set}{\{}{\}}

\DeclarePairedDelimiter{\abs}{\vert}{\vert}

\theoremstyle{plain}
\newtheorem{theorem}{Theorem}
\newtheorem{proposition}[theorem]{Proposition}

\newtheorem{lemma}[theorem]{Lemma}
\theoremstyle{remark}

\newtheorem{example}{Example}

\theoremstyle{definition}
\newtheorem{definition}{Definition}

\begin{document}
\title{Lexicographic choice functions}

 \author{Arthur Van Camp}
 \address{Ghent University\\IDLab\\Technologiepark--Zwijnaarde 914\\9052 Zwijnaarde, Belgium}
 \email{arthur.vancamp@ugent.be}
 \author{Gert de Cooman}
 \address{Ghent University\\IDLab\\Technologiepark--Zwijnaarde 914\\9052 Zwijnaarde, Belgium}
 \email{gert.decooman@ugent.be}
 \author{Enrique Miranda}
 \address{Department of Statistics and Operations Research\\University of Oviedo, Spain}
 \email{mirandaenrique@uniovi.es}

\begin{abstract}
We investigate a generalisation of the coherent choice functions considered by~\cite{seidenfeld2010}, by sticking to the convexity axiom but imposing no Archimedeanity condition.
We define our choice functions on vector spaces of options, which allows us to incorporate as special cases both \citeauthor{seidenfeld2010}'s \citeyearpar{seidenfeld2010} choice functions on horse lotteries and sets of desirable gambles \citep{quaeghebeur2012:itip}, and to investigate their connections.

We show that choice functions based on sets of desirable options (gambles) satisfy Seidenfeld's convexity axiom only for very particular types of sets of desirable options, which are in a one-to-one relationship with the lexicographic probabilities. 
We call them lexicographic choice functions. 
Finally, we prove that these choice functions can be used to determine the most conservative convex choice function associated with a given binary relation.
\end{abstract}


\maketitle

\section{Introduction}
Since the publication of the seminal work of \citet{arrow1951} and \citet{uzawa1956}, coherent choice functions have been used widely as a model of the rational behaviour of an individual or a group.
In particular, \citet{seidenfeld2010} established an axiomatisation of coherent choice functions, generalising \citeauthor{Rubin1987}'s \citeyearpar{Rubin1987} axioms to allow for incomparability.
Under this axiomatisation, they proved a representation theorem for coherent choice functions in terms of probability-utility pairs: a choice function $\cf$ satisfies their coherence axioms if and only if there is some non-empty set $S$ of probability-utility pairs such that $f\in\cf(\os)$ whenever the option $f$ maximises $p$-expected $u$-utility over the set of options $\os$ for some $(p,u)$ in $S$.

Allowing for incomparability between options may often be of crucial importance.
Faced with a choice between two options, a subject may not have enough information to establish a (strict or weak) preference of one over the other: the two options may be incomparable.
This will indeed typically be the case when the available information is too vague or limited.
It arises quite intuitively for group decisions, but also for decisions made by a single subject, as was discussed quite thoroughly by \citet{williams1975}, \citet{levi1980}, and \citet{walley1991}, amongst many others.
Allowing for incomparability lies at the basis of a generalising approach to probability theory that is often referred to by the term \emph{imprecise probabilities}.
It unifies a diversity of well-known uncertainty models, including typically non-linear (or non-additive) functionals, credal sets, and sets of desirable gambles; see the introductory book by \citet{augustin2011} for a recent overview.
Among these, coherent sets of desirable gambles, as discussed by \citet{quaeghebeur2012:itip}, are usually considered to constitute the most general and powerful type of model.
Such sets collect the gambles that a given subject considers strictly preferable to the status quo.

Nevertheless, choice functions clearly lead to a still more general model than sets of desirable gambles, because the former's preferences are not necessarily completely determined by the pair-wise comparisons between options that essentially constitute the latter.
This was of course already implicit in \citeauthor{seidenfeld2010}'s \citeyearpar{seidenfeld2010} work, but was investigated in detail in one of our recent papers~\citep{vancamp2017}, where we zoomed in on the connections between choice functions, sets of desirable gambles, and indifference.

In order to explore the connection between indifference and the strict preference expressed by choice functions, we extended the above-mentioned axiomatisation by \citet{seidenfeld2010} to choice functions defined on vector spaces of options, rather than convex sets of horse lotteries, and also let go of two of their axioms: (i) the Archi\-me\-dean one, because it prevents choice functions from modelling the typically non-Archimedean preferences captured by coherent sets of desirable gambles; and (ii) the convexity axiom, because it turns out to be hard to reconcile with Walley--Sen maximality as a decision rule, something that is closely tied in with coherent sets of desirable options~\citep{troffaes2007}.
Although our alternative axiomatisation allows for more leeway, and for an easy comparison with the existing theory of sets of desirable gambles, it also has the drawback of no longer forcing a Rubinesque representation theorem, or in other words, of not leading to a strong belief structure \Citep{cooman2003a}.
Such a representation is nevertheless interesting, because it allows choice functions to be constructed using basic building blocks.
In an earlier paper~\citep{vancamp2017}, we did discuss a few interesting examples of special `representable' choice functions, such as the ones from a coherent set of desirable gambles via maximality, or those determined by a set of probability measures via E-admissibility.

The goal of the present paper is twofold: to (i) further explore the connection of our definition of choice functions with \citeauthor{seidenfeld2010}'s \citeyearpar{seidenfeld2010}; and to (ii) investigate in detail the implications of \citeauthor{seidenfeld2010}'s \citeyearpar{seidenfeld2010} convexity axiom in our context.
We will prove that, perhaps somewhat surprisingly, for those choice functions that are uniquely determined by binary comparisons, convexity is equivalent to being representable by means of a lexicographic probability measure.
This is done by first establishing the implications of convexity in terms of the binary comparisons associated with a choice function, giving rise to what we will call \emph{lexicographic sets of desirable gambles}.
These sets include as particular cases the so-called \emph{maximal} and \emph{strictly desirable} sets of desirable gambles.
Although in the particular case of binary possibility spaces these are the only two possibilities, for more general spaces lexicographic sets of gambles allow for a greater level of generality, as one would expect considering the above-mentioned equivalence.

A consequence of our equivalence result is that we can consider infima of choice functions associated with lexicographic probability measures, and in this manner subsume the examples of E-admissibility and M-admissibility discussed in an earlier paper~\citep{vancamp2017}.
It will follow from the discussion that these infima also satisfy the convexity axiom.
As one particularly relevant application of these ideas, we prove that the most conservative convex choice function associated with a binary preference relation can be obtained as the infimum of its dominating lexicographic choice functions.

The paper is organised as follows.
In Section~\ref{sec:coherent-choice}, we recall the basics of coherent choice functions on vector spaces of options as introduced in our earlier work~\citep{vancamp2015}.
We motivate our definitions by showing in Section~\ref{sec:connection horse lotteries} that they include in particular coherent choice functions on horse lotteries, considered by \citeauthor{seidenfeld2010}'s \citeyearpar{seidenfeld2010}, and we discuss in some detail the connection between the rationality axioms considered by \citet{seidenfeld2010} and ours.

As a particularly useful example, we discuss in Section~\ref{sec:link with desirability} the choice functions that are determined by binary comparisons. 
We have already shown before~\citep{vancamp2017} that this leads to the model of coherent sets of desirable gambles; here we study the implications of including convexity as a rationality axiom.

In Section~\ref{sec:lexicographic}, we motivate our definition of lexicographic choice functions and study the properties of their associated binary preferences. 
We prove the connection with lexicographic probability systems and show that the infima of such choice functions can be used when we want to determine the implications of imposing convexity and maximality. 
We conclude with some additional discussion in Section~\ref{sec:conclusions}.

\section{Coherent choice functions on vector spaces}
\label{sec:coherent-choice}
Consider a real vector space $\vs$ provided with the vector addition $+$ and scalar multiplication.
We denote the additive identity by $0$.
For any subsets~$\os[1]$ and~$\os[2]$ of $\vs$ and any $\lambda$ in $\reals$, we let $\lambda\os[1]\coloneqq\cset{\lambda\vect}{\vect\in\os[1]}$ and $\os[1]+\os[2]\coloneqq\cset{\vect+\vect[v]}{\vect\in\os[1]\text{ and }\vect[v]\in\os[2]}$.

Elements of $\vs$ are intended as abstract representations of \emph{options} amongst which a subject can express his preferences, by specifying choice functions.
Often, options will be real-valued maps on some possibility space, interpreted as uncertain rewards---and therefore also called \emph{gambles}.
More generally, they can be \emph{vector-valued gambles}: vector-valued maps on the possibility space.
We will see further on that by using such vector-valued gambles, we are able to include as a special case \emph{horse lotteries}, the options considered for instance by~\citet{seidenfeld2010}.
Also, we have shown~\citep{vancamp2017} that indifference for choice functions can be studied efficiently by also allowing equivalence classes of indifferent gambles as options; these yet again constitute a vector space, where now the vectors cannot always be identified easily with maps on some possibility space, or gambles.
For these reasons, we allow in general any real vector space to serve as an our set of (abstract) possible options.
We will call such a real vector space an \emph{option space}.

We denote by $\cfdom\group{\vs}$ the set of all non-empty \emph{finite} subsets of $\vs$, a strict subset of the power set of $\vs$.
When it is clear what option space $\vs$ we are considering, we will also use the simpler notation $\cfdom$.
Elements $\os$ of $\cfdom$ are the option sets amongst which a subject can choose his preferred options.

\begin{definition}\label{def: choice function}
A \emph{choice function} $\cf$ on an option space $\vs$ is a map
\begin{equation*}
\cf\colon\cfdom\to\cfdom\cup\set{\emptyset}\colon\os\mapsto\cf(\os)\text{ such that $\cf(\os)\subseteq\os$.}
\end{equation*}
We collect all the choice functions on $\vs$ in $\allcfs\group{\vs}$, often denoted as $\allcfs$ when it is clear from the context what the option space is.
\end{definition}
\noindent
The idea underlying this simple definition is that a choice function~$\cf$ selects the set $\cf(\os)$ of `best' options in the \emph{option set}~$\os$.
Our definition resembles the one commonly used in the literature~\citep{aizerman1985,seidenfeld2010,sen1977}, except perhaps for an also not entirely unusual restriction to \emph{finite} option sets~\citep{He2012,Schwartz1972,Sen1971}.

Equivalently to a choice function $\cf$, we may consider its associated \emph{rejection function} $\rf$, defined by $\rf(\os)\coloneqq\os\setminus\cf(\os)$ for all $\os$ in $\cfdom$.
It returns the options $\rf(A)$ that are rejected---not selected---by $\cf$.

Another equivalent notion is that of a \emph{choice relation}.
Indeed, for any choice function $\cf$---and therefore for any rejection function $\rf$---the associated choice relation is the binary relation $\crel$ on $\cfdom$~\citep[Section~3]{seidenfeld2010}, defined by:
\begin{equation}\label{eq:choice relation}
\os[1]\crel\os[2]
\iff
\os[1]\subseteq\rf\group{\os[1]\cup\os[2]}
\text{ for all $\os[1]$ and $\os[2]$ in $\cfdom$.}
\end{equation}
The intuition behind $\crel$ is clear: $\os[1]\crel\os[2]$ whenever every option in $\os[1]$ is rejected when presented with the options in $\os[1]\cup\os[2]$.

\subsection{Useful basic definitions and notation}
We call $\nats$ the set of all (positive) integers, 
$\posreals$ the set of all (strictly) positive real numbers, and $\nonnegreals\coloneqq\posreals\cup\set{0}$.

Given any subset $\os$ of an option space $\vs$, we define its \emph{positive hull} $\Posi(\os)$ as the set of all positive finite linear combinations of elements of $\os$:
\begin{equation*}
\Posi\group{\os}\coloneqq\cset[\bigg]{\sum_{k=1}^n\lambda_k\vect[u_k]}{n\in\nats,\lambda_k\in\posreals,\vect[u_k]\in\os}\subseteq\vs,
\end{equation*}
and its \emph{convex hull} $\ch\group{\os}$ as the set of convex combinations of elements of $\os$:
\begin{equation*}
\ch\group{\os}\coloneqq\cset[\bigg]{\sum_{k=1}^n\alpha_k\vect[u_k]} {n\in\nats,\alpha_k\in\nonnegreals,\sum_{k=1}^n\alpha_k=1,\vect[u_k]\in\os}\subseteq\Posi\group{\os}\subseteq\vs.
\end{equation*}

A subset $\os$ of $\vs$ is called a \emph{convex cone} if it is closed under positive finite linear combinations, i.e.~if $\Posi\group{\os}=\os$.
A convex cone $\cone$ is called \emph{proper} if $\cone\cap-\cone=\set{0}$.

With any proper convex cone $\cone\subseteq\vs$, we can associate an ordering~$\vo[\cone]$ on~$\vs$, defined for all $\vect$ and $\vect[v]$ in~$\vs$ as follows:
\begin{equation*}
\vect\vo[\cone]\vect[v]
\iff\vect[v]-\vect\in\cone.
\end{equation*}
We also write $\vect\voi[\cone]\vect[v]$ for $\vect[v]\vo[\cone]\vect$.
The ordering $\vo[\cone]$ is actually a \emph{vector ordering}: it is a partial order---reflexive, antisymmetric and transitive---that satisfies the following two characteristic properties:
\begin{align}
&\vect[u_1]\vo[\cone]\vect[u_2]
\iff\vect[u_1]+\vect[v]\vo[\cone]\vect[u_2]+\vect[v]
\label{eq: vector ordering 1: add constant vector to both sides};\\
&\vect[u_1]\vo[\cone]\vect[u_2]
\iff\lambda\vect[u_1]\vo[\cone]\lambda\vect[u_2],
\label{eq: vector ordering 2: multiply with positive lambda},
\end{align}
for all $\vect[u_1]$, $\vect[u_2]$ and $\vect[v]$ in $\vs$, and all $\lambda$ in $\posreals$.
Observe, by the way, that as a consequence
\begin{equation*}
\vect\vo[\cone]\vect[v]
\iff0\vo[\cone]\vect[v]-\vect
\iff\vect-\vect[v]\vo[\cone]0
\end{equation*}
for all $\vect$ and $\vect[v]$ in $\vs$.

Conversely, given any vector ordering $\vo$, the proper convex cone $\cone$ from which it is derived can always be retrieved by $\cone=\cset{\vect\in\vs}{\vect\voi0}$.
When the abstract options are gambles, $\vo$ will typically be the point-wise order $\leq$, but it need not necessarily be.

Finally, with any vector ordering $\vo$, we associate the strict partial ordering $\svo$ as follows:
\begin{equation*}
\vect\svo\vect[v]
\iff\group{\vect\vo\vect[v]\text{ and }\vect\neq\vect[v]}\iff\vect[v]-\vect\in\cone\setminus\set{0}
\text{ for all $\vect$ and $\vect[v]$ in $\vs$.}
\end{equation*}
We call $\vect$ \emph{positive} if $\vect\svoi0$, and collect all positive options in the convex cone~$\vspos\coloneq\cone\setminus\set{0}$.

From now on, we assume that $\vs$ is an \emph{ordered vector space}, with a generic but fixed vector ordering $\vo[\cone]$. We will refrain from explicitly mentioning the actual proper convex cone~$\cone$ we are using, and simply write $\vs$ to mean the ordered vector space, and use $\vo$ as a generic notation for the associated vector ordering.

\subsection{Rationality axioms}\label{sec:rationality axioms}
We focus on a special class of choice functions, which we will call \emph{coherent}.

\begin{definition}\label{def: rationality axioms for choice functions}
We call a choice function~$\cf$ on~$\vs$ \emph{coherent} if for all $\os$, $\os[1]$ and $\os[2]$ in $\cfdom$, all $\vect$ and $\vect[v]$ in $\vs$, and all $\lambda$ in $\posreals$:
\begin{enumerate}[noitemsep,topsep=0pt,label=\upshape C$_{\arabic*}$.,ref=\upshape C$_{\arabic*}$,leftmargin=*]
\item\label{coh cf 1: irreflexivity}
  $\cf\group{\os}\neq\emptyset$;
\item\label{coh cf 2: non-triviality}
  if $\vect\svo\vect[v]$ then $\set{\vect[v]}=\cf\group{\set{\vect,\vect[v]}}$;
\item\label{coh cf 3}
  \begin{enumerate}[noitemsep,leftmargin=*,label=\upshape\alph*.,ref=\theenumi\upshape\alph*]
    \item\label{coh cf 3a: alpha}
      if~$\cf\group{\os[2]}\subseteq\os[2]\setminus\os[1]$ and $\os[1]\subseteq\os[2]\subseteq\os$
      then~$\cf\group{\os}\subseteq\os\setminus\os[1]$;
    \item\label{coh cf 3b: aizerman}
      if~$\cf\group{\os[2]}\subseteq\os[1]$ and $\os\subseteq\os[2]\setminus\os[1]$
      then~$\cf\group{\os[2]\setminus\os}\subseteq\os[1]$;
  \end{enumerate}
\item\label{coh cf 4}
  \begin{enumerate}[noitemsep,leftmargin=*,label=\upshape\alph*.,ref=\theenumi\upshape\alph*]
    \item\label{coh cf 4a: scaling}
      if~$\os[1]\subseteq\cf\group{\os[2]}$ then~$\lambda\os[1]\subseteq\cf\group{\lambda\os[2]}$;
    \item\label{coh cf 4b: independence}
      if~$\os[1]\subseteq\cf\group{\os[2]}$ then~$\os[1]+\set{\vect}\subseteq\cf\group{\os[2]+\set{\vect}}$;
  \end{enumerate}
\end{enumerate}
We collect all the coherent choice functions on $\vs$ in $\allcohcfs\group{\vs}$, often denoted as $\allcohcfs$ when it is clear from the context what the option space is.
\end{definition}
\noindent
Parts~\ref{coh cf 3a: alpha} and~\ref{coh cf 3b: aizerman} of Axiom~\ref{coh cf 3} are respectively known as \emph{Sen's condition $\alpha$} and \emph{Aizerman's condition}.
They are more commonly written in terms of the rejection function as, respectively:
\begin{equation}\label{eq:senalpha}
\group{\os[1]\subseteq\rf\group{\os[2]}\text{ and } \os[2]\subseteq\os} \then\os[1]\subseteq\rf\group{\os}
\text{, for all $\os,\os[1],\os[2]$ in $\cfdom$,}
\end{equation}
and
\begin{equation}\label{eq:aizerman}
\group{\os[1]\subseteq\rf\group{\os[2]} \text{ and }\os\subseteq\os[1]} \then\os[1]\setminus\os\subseteq\rf\group{\os[2]\setminus\os}
\text{, for all $\os,\os[1],\os[2]$ in $\cfdom$.}
\end{equation}

These axioms constitute a subset of the ones introduced by~\citet{seidenfeld2010}, duly translated from horse lotteries to our abstract options, which are more general as we will show in Section~\ref{sec:connection horse lotteries} further on.
In this respect, our notion of coherence is less restrictive than theirs. 
On the other hand, our Axiom~\ref{coh cf 2: non-triviality} is more restrictive the corresponding one in~\cite{seidenfeld2010}. 
This is necessary for the link between coherent choice functions and coherent sets of desirable gambles we will establish in Section~\ref{sec:link with desirability}. 

One axiom we omit from our coherence definition, is the Archimedean one.
Typically the preference associated with coherent sets of desirable gambles does not have the Archimedean property~\citep[Section~3]{zaffalon2015:birth}, so letting go of this axiom is necessary if we want to explore the connection with desirability.

The second axiom that we do not consider as necessary for coherence is what we will call the \emph{convexity axiom}:
\begin{enumerate}[noitemsep,topsep=0pt,label=\upshape C$_{\arabic*}$.,ref=\upshape C$_{\arabic*}$,leftmargin=*,start=5]
\item\label{coh cf 5: convexity} if $\os\subseteq\os[1]\subseteq\ch\group{\os}$ then $\cf\group{\os}\subseteq\cf\group{\os[1]}$, for all $\os$ and $\os[1]$ in $\cfdom$.
\end{enumerate}
As we will show in Section~\ref{sec:link with desirability}, it is incompatible with Walley--Sen maximality \citep{walley1991,troffaes2007} as a decision rule.
Nevertheless, we intend to investigate the connection with desirability for coherent choice functions that do satisfy the convexity axiom.


Two dominance properties are immediate consequences of coherence:

\begin{proposition}\label{prop:dominance}
Let $\cf$ be a coherent choice function on $\cfdom$.
Then for all $\vect[u_1]$ and $\vect[u_2]$ in $\vs$ such that $\vect[u_1]\vo\vect[u_2]$, all $\os$ in $\cfdom$ and all $\vect[v]$ in $\os\setminus\set{\vect[u_1],\vect[u_2]}$:
\begin{enumerate}[label=\upshape\alph*.,ref=\upshape\alph*,leftmargin=*,noitemsep]
\item\label{prop: consequences of coherence: item 4 a}
if\/ $\vect[u_2]\in\os$ and\/ $\vect[v]\notin\cf(\os\cup\set{\vect[u_1]})$
then $\vect[v]\notin\cf(\os)$;
\item\label{prop: consequences of coherence: item 4 b}
if\/ $\vect[u_1]\in\os$ and\/ $\vect[v]\notin\cf(\os)$ then
$\vect[v]\notin\cf(\set{\vect[u_2]}\cup\os\setminus\set{\vect[u_1]})$.
\end{enumerate}
\end{proposition}

\begin{proof}
The result is trivial when $\vect[u_1]=\vect[u_2]$, so let us assume that $\vect[u_1]\svo\vect[u_2]$.

The first statement is again trivial if $\vect[u_1]\in\os$.
When $\vect[u_1]\notin\os$, it follows from Axiom~\ref{coh cf 2: non-triviality} that $\vect[u_1]\notin\cf\group{\set{\vect[u_1],\vect[u_2]}}$.
By applying Axiom~\ref{coh cf 3a: alpha} in the form of Equation~\eqref{eq:senalpha}, we find that $\vect[u_1]\notin\cf(\os\cup\set{\vect[u_1]})$, and then applying Axiom~\ref{coh cf 3b: aizerman} in the form of Equation~\eqref{eq:aizerman}, together with the assumption that $\vect[v]\notin\cf\group{\os\cup\set{\vect[u_1]}}$, we conclude that $\vect[v]\notin\cf(\os\cup\set{\vect[u_1]}\setminus\set{\vect[u_1]})=\cf(\os)$.

For the second statement, it follows from Axiom~\ref{coh cf 2: non-triviality} that $\vect[u_1]\notin\cf(\set{\vect[u_1],\vect[u_2]})$. 
By applying Axiom~\ref{coh cf 3a: alpha} in the form of Equation~\eqref{eq:senalpha}, we find that both $\vect[u_1]\notin\cf(\os\cup\set{\vect[u_2]})$ and $\vect[v]\notin\cf(\os\cup\set{\vect[u_2]})$, so we can apply Axiom~\ref{coh cf 3b: aizerman}  in the form of Equation~\eqref{eq:aizerman} to conclude that $\vect[v]\notin\cf(\set{\vect[u_2]}\cup\os\setminus\set{\vect[u_1]})$.
\end{proof}

We are interested in conservative reasoning with choice functions.
We therefore introduce a binary relation $\nmi$ on the set $\allcfs$ of all choice functions, having the interpretation of `not more informative than', or, in other words, `at least as uninformative as'.

\begin{definition}\label{def: not more informative than relation}
Given two choice functions $\cf[1]$ and $\cf[2]$ in $\allcfs$, we call $\cf[1]$ \emph{not more informative than} $\cf[2]$---and we write $\cf[1]\nmi\cf[2]$---if $\group{\forall\os\in\cfdom}\cf[1](\os)\supseteq\cf[2](\os)$.
\end{definition}
\noindent
This intuitive way of ordering choice functions is also used by \citet[Section~2]{Bradley2015} and \citet[Definition~6]{vancamp2017}.
The underlying idea is that a choice function is more informative when it consistently chooses more specifically---or more restrictively---amongst the available options.

Since, by definition, $\nmi$ is a product ordering of set inclusions, the following result is immediate~\citep{davey1990}.

\begin{proposition}\label{prop: the not more informative relation leads to posets}
The structure $\poset{\allcfs}{\nmi}$ is a complete lattice:
\begin{enumerate}[leftmargin=*,noitemsep,topsep=0pt,label=\upshape(\roman*)]
\item it is a \emph{partially ordered set}, or \emph{poset}, meaning that the binary relation $\nmi$ on $\allcfs$ is \emph{reflexive}, \emph{antisymmetric} and \emph{transitive};
\item for any subset\/ $\cfs$ of\/ $\allcfs$, its infimum $\inf\cfs$ and its supremum $\sup\cfs$ with respect to the ordering\/ $\nmi$ exist in $\allcfs$, and are given by $\inf\cfs\group{\os}=\bigcup_{\cf\in\cfs}\cf\group{\os}$ and $\sup\cfs\group{\os}=\bigcap_{\cf\in\cfs}\cf\group{\os}$ for all $\os$ in $\cfdom$.
\end{enumerate}
\end{proposition}
\noindent
The idea underlying these notions of infimum and supremum is that $\inf\cfs$ is the most informative model that is not more informative than any of the models in $\cfs$, and $\sup\cfs$ the least informative model that is not less informative than any of the models in~$\cfs$.

We have proved elsewhere~\citep[Proposition~3]{vancamp2017} that coherence is preserved under arbitrary non-empty infima. 
Because of our interest in the additional Axiom~\ref{coh cf 5: convexity},
we prove that it also is preserved under arbitrary non-empty infima.

\begin{proposition}\label{prop:inf of C5}
Given any non-empty collection $\cfs$ of choice functions that satisfy Axiom~\ref{coh cf 5: convexity}, its infimum $\inf\cfs$ satisfies Axiom~\ref{coh cf 5: convexity} as well.
\end{proposition}

\begin{proof}
Denote $\cf'\coloneqq\inf\cfs$.
Consider any $\os$ and $\os[1]$ in $\cfdom$ such that $\os\subseteq\os[1]\subseteq\ch\group{\os}$.
Then $\cf\group{\os}\subseteq\cf\group{\os[1]}$ for all $\cf$ in $\cfs$, whence $\cf'\group{\os}=\bigcup_{\cf\in\cfs}\cf\group{\os}\subseteq\bigcup_{\cf\in\cfs}\cf\group{\os[1]}=\cf'\group{\os[1]}$.
\end{proof}

\section{The connection with other definitions of choice functions}
\label{sec:connection horse lotteries}
Before we go on with our exploration of choice functions, let us take some time here to explain why we have chosen to define them in the way we did.
\citet{seidenfeld2010} \citep[see also][]{Kadane2004} define choice functions on \emph{horse lotteries}, instead of options, as this helps them generalise the framework of \citet{anscombe1963} for binary preferences to non-binary ones.

One reason for our working with the more abstract notion of options---elements of some general vector space---is that they are better suited for dealing with indifference: this involves working with equivalence classes of options, which again constitute a vector space~\citep{vancamp2017}.
These equivalence classes can no longer be interpreted easily or directly as gambles, or horse lotteries for that matter.
Another reason for using options that are more general than real-valued gambles is that recent work by \citet{zaffalon2015:birth} has shown that a very general theory of binary preference can be constructed using vector-valued gambles, rather than horse lotteries.
Such vector-valued gambles again constitute a real vector, or option, space.
Here, we show that the conclusions of their work can be extended from binary preferences to choice functions.

We consider an arbitrary possibility space $\ps$ of mutually exclusive elementary events, one of which is guaranteed to occur.
Consider also a countable set $\rewards$ of prizes, or rewards.

\begin{definition}[Gambles]
Any bounded real-valued function on some domain $\ps$ is called a \emph{gamble} on $\ps$.
We collect all gambles on $\ps$ in $\gambleson$, often denoted as $\gambles$ when it is clear from the context what the domain $\ps$ is.

When the domain is of the type $\ps\times\rewards$, we call elements $f$ of $\gambleson[\ps\times\rewards]$ \emph{vector-valued gambles} on $\ps$.
Indeed, for each $x$ in $\ps$, the partial map $f(x,\cdot)$ is then an element of the vector space $\gambleson[\rewards]$.
\end{definition}
\noindent
The set $\gambles$, provided with the point-wise addition of gambles, the point-wise multiplication with real scalars, and the point-wise vector ordering $\leq$, constitutes an ordered vector space.
We call $\posgambles\coloneqq\cset{f\in\gambles}{f>0}=\cset{f\in\gambles}{f\geq0\text{ and }f\neq0}$ the set of all \emph{positive} (vector-valued) gambles.

Horse lotteries are special vector-valued gambles.

\begin{definition}[Horse lotteries]
We call \emph{horse lottery} $\hl$ any map from $\ps\times\rewards$ to $[0,1]$ such that for all $x$ in $\ps$, the partial map $\hl(\rval,\cdot)$ is a probability mass function over $\rewards$:
\begin{equation*}
\group{\forall\rval\in\ps}
\group[\bigg]{\sum_{\rew\in\rewards}\hl(\rval,\rew)=1\text{ and }\group{\forall\rew\in\rewards}\hl(\rval,\rew)\geq0}.
\end{equation*}
We collect all the horse lotteries on $\ps$ with reward set $\rewards$ in $\hls\group{\ps,\rewards}$, which is also denoted more simply by $\hls$ when it is clear from the context what the possibility space $\ps$ and reward set $\rewards$ are.
\end{definition}

Let us, for the remainder of this section, fix $\ps$ and $\rewards$.
It is clear that $\hls\subseteq\gambles\group{\ps\times\rewards}$.
\citet{seidenfeld2010} consider choice functions whose domain is $\cfdom\group{\hls}$, the set of all finite subsets of $\hls$---choice functions on horse lotteries.\footnote{Actually, \citet{seidenfeld2010} define choice functions on a larger domain: all possibly infinite but \emph{closed} sets of horse lotteries (non-closed sets may not have admissible options). This is a complication we see no need for in the present context.
}
We will call them \emph{choice functions on $\hls$}.
Because of the nature of $\hls$, their choice functions are different from ours: they require slightly different rationality axioms.
The most significant change is that for \citet{seidenfeld2010}, choice functions need not satisfy Axioms~\ref{coh cf 4a: scaling} and~\ref{coh cf 4b: independence}.
In fact, choice functions on $\hls$ cannot satisfy these axioms, since $\hls$ is no linear space: it is not closed under arbitrary linear combinations, only under \emph{convex combinations}.
Instead, on their approach a choice function $\cf^*$ on $\hls$ is required to satisfy
\begin{enumerate}[noitemsep,topsep=0pt,label=\upshape C$_{\arabic*}^*$.,ref=\upshape C$_{\arabic*}^*$,leftmargin=*,start=4]
\item\label{coh cf 4 Seidenfeld}
$\os[1]^*\crel[{\cf^*}]\os[2]^*
\iff
\alpha\os[1]^*+(1-\alpha)\set{\hl}\crel[{\cf^*}]\alpha\os[2]^*+(1-\alpha)\set{\hl}$
for all $\alpha$ in $(0,1]$, all $\os[1]^*$ and $\os[2]^*$ in $\cfdom\group{\hls}$ and all $\hl$ in~$\hls$.
\end{enumerate}
The binary relation $\crel[{\cf^*}]$ is the choice relation associated with $\cf^*$, defined by Equation~\eqref{eq:choice relation}.
Furthermore, for a choice function $\cf^*$ to be coherent, it needs to additionally satisfy~(see \citep{seidenfeld2010}):
\begin{enumerate}[noitemsep,topsep=0pt,label=\upshape C$_{\arabic*}^*$.,ref=\upshape C$_{\arabic*}^*$,leftmargin=*,start=1]
\item\label{coh cf 1 Seidenfeld} $\cf^*\group{\os^*}\neq\emptyset$ for all $\os^*$ in $\cfdom\group{\hls}$;
\item\label{coh cf 2 Seidenfeld} for all $\os^*$ in $\cfdom\group{\hls}$, all $\hl[1]$ and $\hl[2]$ in $\hls$ such that $\hl[1]\group{\cdot,\bestr}\vo\hl[2]\group{\cdot,\bestr}$ and $\hl[1]\group{\cdot,\rew}=\hl[2]\group{\cdot,\rew}=0$ for all $\rew$ in $\rewards\setminus\set{\worstr,\bestr}$, and all $\hl$ in $\hls\setminus\set{\hl[1],\hl[2]}$:
\begin{enumerate}[noitemsep,leftmargin=*,label=\upshape\alph*.,ref=\theenumi\upshape\alph*]
\item\label{coh cf 2a Seidenfeld}
if~$\hl[2]\in\os^*$ and~$\hl\in\rf^*\group{\set{\hl[1]}\cup\os^*}$ then $\hl\in\rf^*\group{\os^*}$;
\item\label{coh cf 2b Seidenfeld}
if~$\hl[1]\in\os^*$ and~$\hl\in\rf^*\group{\os^*}$ then $\hl\in\rf^*\group{\set{\hl[2]}\cup\os^*\setminus\set{\hl[1]}}$;
\end{enumerate}
\item\label{coh cf 3 Seidenfeld} for all $\os^*$, $\os[1]^*$ and $\os[2]^*$ in $\cfdom\group{\hls}$:
\begin{enumerate}[noitemsep,leftmargin=*,label=\upshape\alph*.,ref=\theenumi\upshape\alph*]
\item\label{coh cf 3a Seidenfeld}
if~$\os[1]^*\subseteq\rf^*\group{\os[2]^*}$ and $\os[2]^*\subseteq\os^*$ then~$\os[1]^*\subseteq\rf^*\group{\os}$;
\item\label{coh cf 3b Seidenfeld}
if~$\os[1]^*\subseteq\rf^*\group{\os[2]^*}$ and $\os^*\subseteq\os[1]^*$ then $\os[1]^*\setminus\os^*\subseteq\rf^*\group{\os[2]^*\setminus\os}$;
\end{enumerate}
\setcounter{enumi}{4}
\item\label{coh cf 5 Seidenfeld}
if $\os^*\subseteq\os[1]^*\subseteq\ch\group{\os}$ then $\cf^*\group{\os}\subseteq\cf^*\group{\os[1]^*}$, for all $\os^*$ and $\os[1]^*$ in $\cfdom\group{\hls}$;
\item\label{coh cf 6 Seidenfeld} for all $\os^*$, ${\os^*}'$, ${\os^*}''$ ${\os[i]^*}'$ and ${\os[i]^*}''$ (for $i$ in $\nats$) in $\cfdom\group{\hls}$ such that the sequence ${\os[i]^*}'$ converges point-wise to ${\os^*}'$ and the sequence ${\os[i]^*}''$ converges point-wise to ${\os^*}''$:
\begin{enumerate}[noitemsep,leftmargin=*,label=\upshape\alph*.,ref=\theenumi\upshape\alph*]
\item\label{coh cf 6a Seidenfeld}
If $\group{\forall i\in\nats}{\os[i]^*}''\crel[{\cf^*}]{\os[i]^*}'$ and ${\os^*}'\crel[{\cf^*}]\os^*$ then ${\os^*}''\crel[{\cf^*}]\os^*$;
\item\label{coh cf 6b Seidenfeld}
If $\group{\forall i\in\nats}{\os[i]^*}''\crel[{\cf^*}]{\os[i]^*}'$ and $\os^*\crel[{\cf^*}]{\os^*}''$ then $\os^*\crel[{\cf^*}]{\os^*}'$,
\end{enumerate}
\end{enumerate}
where \citet{seidenfeld2010} assume that there is a a unique worst reward $\worstr$ and a unique best reward $\bestr$ in $\rewards$.
This is a somewhat stronger assumption than we will make: further on in this section, we will only assume that there is a unique worst reward.
Axiom~\ref{coh cf 2 Seidenfeld} is the counterpart of Proposition~\ref{prop:dominance} for choice functions on horse lotteries, which is a result of our Axioms~\ref{coh cf 1: irreflexivity}--\ref{coh cf 4}.
\citet{seidenfeld2010} need to impose this property as an axiom, essentially because of the absence in their system of a counterpart for our Axiom~\ref{coh cf 2: non-triviality}.
Axioms~\ref{coh cf 6a Seidenfeld} and~\ref{coh cf 6b Seidenfeld} are Archimedean axioms, hard to reconcile with desirability \citep[see for instance][Section~4]{zaffalon2015:birth}, which is why will not enforce them here.

We now intend to show that under very weak conditions on the rewards set $\rewards$, choice functions on horse lotteries that satisfy \ref{coh cf 4 Seidenfeld} are in a one-to-one correspondence with choice functions on a suitably defined option space that satisfy Axioms~\ref{coh cf 4a: scaling} and~\ref{coh cf 4b: independence}.

Let us first study the impact of Axiom~\ref{coh cf 4 Seidenfeld}.
We begin by showing that an assessment of $\hl\in\cf\group{\os}$ for some $\os$ in $\cfdom\group{\hls}$ implies other assessments of this type.

\begin{proposition}\label{prop:C on horse lotteries}
Consider any choice function $\cf^*$ on $\cfdom\group{\hls}$ that satisfies Axiom~\ref{coh cf 4 Seidenfeld}, any option sets $\os^*$ and ${\os^*}'$ in $\cfdom\group{\hls}$, and any $\hl$ in $\os^*$ and $\hl'$ in ${\os^*}'$.
If there are $\lambda$ and $\lambda'$ in $\posreals$ such that $\lambda\group{\os^*-\set{\hl}}=\lambda'\group{{\os^*}'-\set{\hl'}}$, then
\begin{equation*}
\hl\in\cf^*\group{\os}
\iff
\hl'\in\cf^*\group{{\os^*}'}.
\end{equation*}
\end{proposition}

\begin{proof} 
Fix $\os^*$ and ${\os^*}'$ in $\cfdom\group{\hls}$, $\hl$ in $\os^*$ and $\hl'$ in ${\os^*}'$, $\lambda$ and $\lambda'$ in $\posreals$, and assume that $\lambda\group{\os^*-\set{\hl}}=\lambda'\group{{\os^*}'-\set{\hl'}}$.
We will show that $\hl\in\rf^*\group{\os^*}\iff\hl'\in\rf^*\group{{\os^*}'}$.
We infer from the assumption that
\begin{equation*}
\frac{\lambda}{\lambda+\lambda'}\os^*+\frac{\lambda'}{\lambda+\lambda'}\set{\hl'}=\frac{\lambda'}{\lambda+\lambda'}{\os^*}'+\frac{\lambda}{\lambda+\lambda'}\set{\hl}.
\end{equation*}
If we call $\alpha\coloneqq\frac{\lambda}{\lambda+\lambda'}$ to ease the notation along, then $1-\alpha=\frac{\lambda'}{\lambda+\lambda'}$ and $\alpha\in(0,1)$.
We now infer from the identity above that $\alpha\os^*+(1-\alpha)\set{\hl'}=(1-\alpha){\os^*}'+\alpha\set{\hl}$.
Consider the following chain of equivalences:
\begin{align*}
\hl\in\rf^*\group{\os^*}
&\iff
\set{\hl}\crel[{\cf^*}]\os^*&&\text{by Equation~\eqref{eq:choice relation}}\\
&\iff
\alpha\set{\hl}+(1-\alpha)\set{\hl'}\crel[{\cf^*}]\alpha\os^*+(1-\alpha)\set{\hl'}&&\text{using Axiom~\ref{coh cf 4 Seidenfeld}}\\
&\iff
\alpha\set{\hl}+(1-\alpha)\set{\hl'}\crel[{\cf^*}](1-\alpha){\os^*}'+\alpha\set{\hl}\\
&\iff
\set{\hl'}\crel[{\cf^*}]{\os^*}'&&\text{using Axiom~\ref{coh cf 4 Seidenfeld}}\\
&\iff
\hl'\in\rf^*\group{{\os^*}'}&&\text{by Equation~\eqref{eq:choice relation}.}
\qedhere
\end{align*}
\end{proof}

For any $\rew$ in $\rewards$, we now introduce $\rewardsw[\rew]\coloneqq\rewards\setminus\set{\rew}$, the set of all rewards without $\rew$.
For the connection between choice functions on $\hls$ and choice functions on some option space, we need to somehow be able to extend $\hls$ to a linear space.
The so-called gamblifier $\poh[\rew]$ will play a crucial role in this:

\begin{definition}[Gamblifier \protect{$\poh[\rew]$}]\label{def:projection operator}
Consider any $\rew$ in $\rewards$.
The \emph{gamblifier} $\poh[\rew]$ is the linear map
\begin{equation*}
\poh[\rew]\colon\gambleson[\ps\times\rewards]\to\gambleson[{\ps\times\rewardsw[\rew]}]
\colon f\mapsto\poh[\rew]f,
\end{equation*}
where $\poh[\rew]f\group{\rval,\rewtoo}\coloneqq f\group{\rval,\rewtoo}$ for all $\rval$ in $\ps$ and $\rewtoo$ in $\rewardsw[\rew]$.
\end{definition}
\noindent
In particular, the gamblifier $\poh[\rew]$ maps any horse lottery $\hl$ in $\hls\group{\ps,\rewards}$ to an element $\poh[\rew]\hl$ of $\gambleson[{\ps\times\rewardsw[\rew]}]$ that satisfies the following two conditions:
\begin{equation}\label{eq:property poh of hl}
\poh[\rew]\hl\group{\cdot,\cdot}\geq0
\text{ and }
\sum_{\rewtoo\in\rewardsw[\rew]}\poh[\rew]\hl\group{\cdot,\rewtoo}\leq1.
\end{equation}
Application of $\poh[\rew]$ to sets of the form $\lambda\group{\os^*-\set{\hl}}$ essentially leaves the `information' they contain unchanged:

\begin{lemma}\label{lemma:proj operator}
Consider any $\rew$ in $\rewards$.
Then the following two properties hold:
\begin{enumerate}[label=\upshape(\roman*),leftmargin=*]
\item The gamblifier $\poh[\rew]$ is one-to-one on $\hls$.
\item For any $\os^*$ and ${\os^*}'$ in $\cfdom\group{\hls}$, any $\hl$ in $\os^*$ and $\hl'$ in ${\os^*}'$ and any $\lambda$ and $\lambda'$ in\/ $\posreals$:
\begin{equation*}
\lambda\group{\os^*-\set{\hl}}=\lambda'\group{{\os^*}'-\set{\hl'}}
\iff
\poh[\rew]\group{\lambda\group{\os^*-\set{\hl}}}=\poh[\rew]\group{\lambda'\group{{\os^*}'-\set{\hl'}}}
\end{equation*}
\end{enumerate}
\end{lemma}

\begin{proof}
We begin with the first statement.
Consider any $\hl$ and $\hl'$ in $\hls$, and assume that $\poh[\rew]\group{\hl}=\poh[\rew]\group{\hl'}$.
We infer from Definition~\ref{def:projection operator} that
\begin{equation*}
\hl\group{\rval,\rewtoo}=\hl'\group{\rval,\rewtoo}
\text{ for all $\rval$ in $\ps$ and $\rewtoo$ in $\rewardsw[\rew]$,}
\end{equation*}
and therefore also, since $\hl$ and $\hl'$ are horse lotteries,
\begin{equation*}
\hl\group{\rval,\rew}
=1-\sum_{\rewtoo\in\rewardsw[\rew]}\hl\group{\rval,\rewtoo}
=1-\sum_{\rewtoo\in\rewardsw[\rew]}\hl'\group{\rval,\rewtoo}
=\hl'\group{\rval,\rew}
\text{ for all $\rval$ in $\ps$.}
\end{equation*}
Hence indeed $\hl=\hl'$.

The direct implication in the second statement is trivial; let us prove the converse.
Assume that $\poh[\rew]\group{\lambda\group{\os^*-\set{\hl}}}=\poh[\rew]\group{\lambda'\group{{\os^*}'-\set{\hl'}}}$.
We may write, without loss of generality, that $\os=\set{\hl,\hl[1],\dots,\hl[n]}$ and ${\os^*}'=\set{\hl',\hl[1]',\dots,\hl[m]'}$ for some $n$ and $m$ in $\nats$. 
Now, consider any element $\hl[i]$ in $\os^*$, then $\poh[\rew]\group{\lambda\group{\hl[i]-\hl}}\in\poh[\rew]\group{\lambda\group{\os^*-\set{\hl}}}$.
Consider any $j$ in $\set{1,\dots,m}$ such that $\poh[\rew]\group{\lambda\group{\hl[i]-\hl}}=\poh[\rew]\group{\lambda'\group{\hl[j]'-\hl'}}$.
It follows from the assumption that there is at least one such $j$.
The proof is complete if we can show that $\lambda\group{\hl[i]-\hl}=\lambda'\group{\hl[j]'-\hl'}$.
By Definition~\ref{def:projection operator}, we already know that
\begin{equation*}
\lambda\group{\hl[i]\group{\cdot,\rewtoo}-\hl\group{\cdot,\rewtoo}}
=
\lambda'\group{\hl[j]'\group{\cdot,\rewtoo}-\hl'\group{\cdot,\rewtoo}}
\text{ for all $\rewtoo$ in $\rewardsw[\rew]$,}
\end{equation*}
and therefore, since $\hl$, $\hl'$, $\hl[i]$ and $\hl[j]'$ are horse lotteries, also
\begin{align*}
\lambda\group{\hl[i]\group{\cdot,\rew}-\hl\group{\cdot,\rew}}
&=\lambda\group[\bigg]{\sum_{\rewtoo\in\rewardsw[\rew]}\hl\group{\cdot,\rewtoo}-\sum_{\rewtoo\in\rewardsw[\rew]}\hl[i]\group{\cdot,\rewtoo}}
=\sum_{\rewtoo\in\rewardsw[\rew]}\lambda\group{\hl\group{\cdot,\rewtoo}-\hl[i]\group{\cdot,\rewtoo}}\\
&=\sum_{\rewtoo\in\rewardsw[\rew]}\lambda'\group{\hl'\group{\cdot,\rewtoo}-\hl[j]'\group{\cdot,\rewtoo}}
=\lambda'\group[\bigg]{\sum_{\rewtoo\in\rewardsw[\rew]}\hl'\group{\cdot,\rewtoo}-\sum_{\rewtoo\in\rewardsw[\rew]}\hl[j]'\group{\cdot,\rewtoo}}\\
&=\lambda'\group{\hl[j]'\group{\cdot,\rew}-\hl'\group{\cdot,\rew}},
\end{align*}
whence indeed $\lambda\group{\hl[i]-\hl}=\lambda'\group{\hl[j]'-\hl'}$.
\end{proof}
\noindent
We now lift the gamblifier $\poh[\rew]$ to a map $\lift[\rew]$ that turns choice functions on gambles into choice functions on horse lotteries:
\begin{equation}\label{eq:lift of poh}
\lift[\rew]
\colon\allcfs\group{\gambleson[{\ps\times\rewardsw[\rew]}]}\to\allcfs\group{\hls\group{\ps,\rewards}}
\colon\cf\mapsto\lift[\rew]\cf,
\end{equation}
where $\lift[\rew]\cf\group{\os^*}\coloneqq\poh[\rew]^{-1}\cf\group{\poh[\rew]\os^*}$ for every $\os^*$ in $\cfdom\group{\hls\group{\ps,\rewards}}$.
This definition makes sense because we have proved in Lemma~\ref{lemma:proj operator} that $\poh[\rew]$ is one-to-one on $\hls$, and therefore invertible on $\poh[\rew]\hls$.
The result of applying $\lift[\rew]$ to a choice function $\cf$ on $\gambleson[{\ps\times\rewardsw[\rew]}]$ is a choice function $\lift[\rew]\cf$ on $\hls\group{\ps,\rewards}$.
Observe that we can equally well make $\lift[\rew]$ apply to rejection functions $\rf$, and that for every $\os^*$ in $\cfdom\group{\hls\group{\ps,\rewards}}$:
\begin{equation*}
\lift[\rew]\rf\group{\os^*}
\coloneqq\poh[\rew]^{-1}\rf\group{\poh[\rew]\os^*}
=\poh[\rew]^{-1}\group{\poh[\rew]\os^*\setminus\cf\group{\poh[\rew]\os^*}}
=\os^*\setminus\poh[\rew]^{-1}\cf\group{\poh[\rew]\os^*}
=\os^*\setminus\lift[\rew]\cf\group{\os^*},
\end{equation*}
so $\lift[\rew]\rf$ is the rejection function associated with the choice function $\lift[\rew]\cf$, when $\rf$ is the rejection function for $\cf$.

One property of the transformation $\lift[\rew]$ that will be useful in our subsequent proofs is the following: 


\begin{lemma}\label{lem:rescaling}
Consider any $\rew$ in $\rewards$ and any $\os$ in $\cfdom\group{\gambles\group{\ps\times\rewardsw[\rew]}}$, and define $g$ by $g\group{\rval,\rewtoo}\coloneqq\sum_{f\in\os}\abs{f\group{\rval,\rewtoo}}$ for all $\rval$ in $\ps$ and $\rewtoo$ in $\rewardsw[\rew]$.
Consider any $\lambda$ in $\reals$ such that
\begin{equation*}
\lambda
>
\max\cset[\bigg]{\max_{\rval\in\ps}\sum_{\rewtoo\in\rewardsw[\rew]}h\group{\rval,\rewtoo}}{h\in A+\set{g}}
\geq0.
\end{equation*}
Then $\frac{1}{\lambda}\group{\os+\set{g}}=\poh[\rew]\os^*$ for some $\os^*$ in $\cfdom\group{\hls\group{\ps,\rewards}}$.
\end{lemma}

\begin{proof}
Consider any $h$ in $\os+\set{g}$, and let us show that $\frac{1}{\lambda}h$ satisfies the conditions in Equation~\eqref{eq:property poh of hl}.
The first one is satisfied because $\lambda>0$ and $h=f+g$ for some $f$ in $\os$, so $h=f+g=f+\sum_{f'\in\os}\abs{f'}\geq f+\abs{f}\geq0$ and therefore indeed $\frac{1}{\lambda}h\geq0$.
For the second condition, recall that $\lambda\geq\sum_{\rewtoo\in\rewardsw[\rew]}h\group{\cdot,\rewtoo}$ by construction and therefore indeed $\sum_{\rewtoo\in\rewardsw[\rew]}\frac{1}{\lambda}h\group{\cdot,\rewtoo}\leq1$.
\end{proof}

\begin{proposition}\label{prop:lift is one-to-one}
Consider any $\rew$ in $\rewards$.
The operator $\lift[\rew]$ is one-to-one on the choice functions on $\gambleson[{\ps\times\rewardsw[\rew]}]$ that satisfy Axioms~\ref{coh cf 4a: scaling} and~\ref{coh cf 4b: independence}.
\end{proposition}

\begin{proof}
Assume \emph{ex absurdo} that $\lift[\rew]$ is not one-to-one, so there are choice functions $\cf$ and $\cf'$ on $\gambles\group{\ps\times\rewards[\rew]}$ that satisfy Axioms~\ref{coh cf 4a: scaling} and~\ref{coh cf 4b: independence}, such that $\lift[\rew]\cf=\lift[\rew]\cf'$ but nevertheless $\cf\neq\cf'$.
The latter means that there are $\os$ in $\cfdom\group{\gambleson[{\ps\times\rewardsw[\rew]}]}$ and $f$ in $\os$ such that $f\in\cf\group{\os}$ and $f\notin\cf'\group{\os}$.
Use Lemma~\ref{lem:rescaling} to find some $\lambda$ in $\posreals$ and $g$ in $\gambleson[{\ps\times\rewardsw[\rew]}]$ such that $\frac{1}{\lambda}(A+\set{g})=\poh[\rew]\os^*$ for some $\os^*$ in $\cfdom\group{\hls\group{\ps,\rewards}}$.
If we now apply Axioms~\ref{coh cf 4a: scaling} and~\ref{coh cf 4b: independence} we find that $\frac{f+g}{\lambda}\in\cf\group{\frac{1}{\lambda}\group{\os+\set{g}}}$, 
or equivalently, $\poh[\rew]^{-1}\group{\frac{f+g}{\lambda}}\in\lift[\rew]\cf\group{\os^*}$.
Similarly, we find that $\frac{f+g}{\lambda}\notin\cf'\group{\frac{1}{\lambda}\group{\os+\set{g}}}$, or equivalently, $\poh[\rew]^{-1}\group{\frac{f+g}{\lambda}}\notin\lift[\rew]\cf'\group{\os^*}$.
But this contradicts our assumption that $\lift[\rew]\cf=\lift[\rew]\cf'$.
\end{proof}

Specifying a choice function $\cf^*$ on $\hls$ induces a strict preference relation on the reward set, as follows. 
With any reward $\rew$ in $\rewards$ we can associate the constant and degenerate lottery $\hl[\rew]$ by letting
\begin{equation}\label{eq:degen-hl}
\hl[\rew](\rval,\rewtoo)
\coloneqq
\begin{cases}
1&\text{ if $\rewtoo=\rew$}\\
0&\text{ otherwise}
\end{cases}
\text{ for all $\rval$ in $\ps$ and $\rewtoo$ in $\rewards$}.
\end{equation}
This is the lottery that associates the certain reward $\rew$ with all states. Then a reward $\rew$ is strictly preferred to a reward $\rewtoo$ when $\hl[\rewtoo]\in\rf^*\group{\set{\hl[\rew],\hl[\rewtoo]}}$.

\begin{definition}[$\cf^*$ has worst reward $\rew$]
Consider any reward $\rew$ in $\rewards$, and any choice function $\cf^*$ on $\hls\group{\ps,\rewards}$.
We say that \emph{$\cf^*$ has worst reward $\rew$} if $\rew$ is the unique reward in $\rewards$ for which $\hl[\rew]\in\rf^*\group{\set{\hl,\hl[\rew]}}$ for all $\hl$ in $\hls\group{\ps,\rewards}\setminus\set{\hl[\rew]}$.
\end{definition}
\noindent
The notion of \emph{having worst reward} is closely related with what would be the natural translation of Axiom~\ref{coh cf 2: non-triviality} to choice functions $\cf^*$ on $\hls\group{\ps,\rewards}$: if $\cf^*$ satisfies
\begin{equation}\label{eq: cf 2 Seidenfeld alternative}
\group{\forall\hl[1],\hl[2]\in\hls}
\group[\Big]{
\group[\big]{
\hl[1]\neq\hl[2]
\text{ and }
\group{\forall\rewtoo\in\rewards[\rew]}
\group{\hl[1]\group{\cdot,\rewtoo}\leq\hl[2]\group{\cdot,\rewtoo}}}
\then\hl[1]\in\rf^*\group{\set{\hl[1],\hl[2]}}}
\end{equation}
for some $\rew$ in $\rewards$, then we say that \emph{$\cf^*$ satisfies the dominance relation for worst reward $\rew$}.

\begin{proposition}\label{prop:cf 2 and worst element}
Consider any $\rew$ in $\rewards$ and any choice function $\cf$ on $\gambleson[{\ps\times\rewardsw[\rew]}]$.
Then $\lift[\rew]\cf$ satisfies the dominance relation for worst reward $\rew$ (Equation~\eqref{eq: cf 2 Seidenfeld alternative}) if and only if $\lift[\rew]\cf$ has worst reward $\rew$.
\end{proposition}

\begin{proof}
For the direct implication, consider any $\hl$ in $\hls\group{\ps,\rewards}\setminus\set{\hl[\rew]}$.
Then $\hl[\rew]\group{\cdot,\rewtoo}=0\leq\hl\group{\cdot,\rewtoo}$ for all $\rewtoo$ in $\rewardsw[\rew]$, and also $\hl\neq\hl[\rew]$, whence indeed $\hl[\rew]\in\lift[\rew]\rf\group{\set{\hl,\hl[\rew]}}$, because by assumption $\lift[\rew]\cf$ satisfies Equation~\eqref{eq: cf 2 Seidenfeld alternative} for $\rew$.

For the converse implication, consider any $\hl[1]$ and $\hl[2]$ in $\hls\group{\ps,\rewards}$ such that $\hl[1]\neq\hl[2]$ and  $\hl[1]\group{\cdot,\rewtoo}\leq\hl[2]\group{\cdot,\rewtoo}$ for all $\rewtoo$ in $\rewardsw[\rew]$.
Then $\poh[\rew]\hl[1]<\poh[\rew]\hl[2]$, whence $0<\poh[\rew]\group{\hl[2]-\hl[1]}$.
Observe that for the horse lottery $\hl'$ in $\hls\group{\ps,\rewards}$ defined by
\begin{equation*}
\hl'\group{\cdot,\rewtoo}
\coloneqq
\begin{cases}
\hl[2]\group{\cdot,\rewtoo}-\hl[1]\group{\cdot,\rewtoo}
&\text{ if $\rewtoo\in\rewardsw[\rew]$}\\
1-\sum_{\rewtoo\in\rewardsw[\rew]}\group[\big]{\hl[2]\group{\cdot,\rewtoo}-\hl[1]\group{\cdot,\rewtoo}}
&\text{ if $\rewtoo=\rew$},
\end{cases}
\end{equation*}
we have that $\poh[\rew]\hl'=\poh[\rew]\group{\hl[2]-\hl[1]}$.
Because $\lift[\rew]\cf$ is assumed to have worst reward $\rew$, we know that in particular $\hl[\rew]\in\lift[\rew]\rf\group{\set{\hl',\hl[\rew]}}$, so we infer from Equation~\eqref{eq:lift of poh} that $0=\poh[\rew]\hl[\rew]\in\rf\group{\set{\poh[\rew]\hl[\rew],\poh[\rew]\hl'}}=\rf\group{\set{0,\poh[\rew]\hl[2]-\poh[\rew]\hl[1]}}$.
Now use Axiom~\ref{coh cf 4b: independence} to infer that $\poh[\rew]\hl[1]\in\rf\group{\set{\poh[\rew]\hl[1],\poh[\rew]\hl[2]}}$, whence indeed $\hl[1]\in\lift[\rew]\rf\group{\set{\hl[1],\hl[2]}}$, by Equation~\eqref{eq:lift of poh}.
\end{proof}

Applying the lifting $\lift[\rew]$ furthermore preserves coherence:

\begin{theorem}
Consider any reward $\rew$ in $\rewards$, and any choice function $\cf$ on $\gambles\group{\ps\times\rewardsw[\rew]}$ that satisfies Axioms~\ref{coh cf 4a: scaling} and\/~\ref{coh cf 4b: independence}.
Then the following statements hold:
\begin{enumerate}[label=\upshape(\roman*),leftmargin=*]
\item\label{it:equivalence coherence 1} $\cf$ satisfies Axiom~\ref{coh cf 1: irreflexivity} if and only if\/ $\lift[\rew]\cf$ satisfies Axiom~\ref{coh cf 1 Seidenfeld};
\item\label{it:equivalence coherence 2} $\cf$ satisfies Axiom~\ref{coh cf 2: non-triviality}\ if and only if\/ $\lift[\rew]\cf$ has worst reward $\rew$;
\item\label{it:equivalence coherence 3} $\cf$ satisfies Axiom~\ref{coh cf 3a: alpha} if and only if\/ $\lift[\rew]\cf$ satisfies Axiom~\ref{coh cf 3a Seidenfeld};
\item\label{it:equivalence coherence 4} $\cf$ satisfies Axiom~\ref{coh cf 3b: aizerman} if and only if\/ $\lift[\rew]\cf$ satisfies Axiom~\ref{coh cf 3b Seidenfeld};
\item\label{it:equivalence coherence 5} $\lift[\rew]\cf$ satisfies Axiom~\ref{coh cf 4 Seidenfeld};
\item\label{it:equivalence coherence 6} $\cf$ satisfies Axiom~\ref{coh cf 5: convexity} if and only if\/ $\lift[\rew]\cf$ satisfies Axiom~\ref{coh cf 5 Seidenfeld}.
\end{enumerate}
\end{theorem}

\begin{proof}
For the direct implication of~\ref{it:equivalence coherence 1}, assume that $\cf$ satisfies Axiom~\ref{coh cf 1: irreflexivity}.
Consider any $\os^*$ in $\cfdom\group{\hls\group{\ps,\rewards}}$.
Then $\lift[\rew]\cf\group{\os^*}=\poh[\rew]^{-1}\cf\group{\poh[\rew]\os}\neq\emptyset$.

For the converse implication, assume that $\lift[\rew]\cf$ satisfies Axiom~\ref{coh cf 1 Seidenfeld}.
Consider any $\os$ in $\cfdom\group{\gambles\group{\ps\times\rewardsw[\rew]}}$.
By Lemma~\ref{lem:rescaling}, there are $\lambda$ in $\posreals$ and $g$ in $\gambleson[{\ps\times\rewardsw[\rew]}]$ such that $\frac{1}{\lambda}(A+\set{g})=\poh[\rew]\os^*$ for some $\os^*$ in $\cfdom\group{\hls\group{\ps,\rewards}}$.
Applying Axioms~\ref{coh cf 4a: scaling} and~\ref{coh cf 4b: independence} and the definition of~$\lift$ [Equation~\eqref{eq:lift of poh}], we infer that indeed
\begin{equation*}
\cf\group{\os}
=\lambda\cf\group[\big]{\frac{1}{\lambda}\group{\os+\set{g}}}-\set{g}
=\lambda\cf\group{\poh[\rew]\os^*}-\set{g}
=\lambda\poh[\rew]\lift[\rew]\cf\group{\os^*}-\set{g}
\neq\emptyset.
\end{equation*}

For the direct implication of~\ref{it:equivalence coherence 2}, assume that $\cf$ satisfies Axiom~\ref{coh cf 2: non-triviality}.
Consider any $\hl[1]$ and $\hl[2]$ in $\hls\group{\ps,\rewards}$ such that $\hl[1]\neq\hl[2]$ and $\hl[1]\group{\cdot,\rewtoo}\leq\hl[2]\group{\cdot,\rewtoo}$ for all $\rewtoo$ in $\rewardsw[\rew]$.
Then $\poh[\rew]\hl[1]<\poh[\rew]\hl[2]$, so Axiom~\ref{coh cf 2: non-triviality} guarantees that $\poh[\rew]\hl[1]\in\rf\group{\set{\poh[\rew]\hl[1],\poh[\rew]\hl[2]}}$.
Equation~\eqref{eq:lift of poh} now turns this into $\hl[1]\in\lift[\rew]\rf\group{\set{\hl[1],\hl[2]}}$.
Proposition~\ref{prop:cf 2 and worst element} now tells us that $\lift[\rew]\cf$ has worst reward~$\rew$.

For the converse implication, assume that $\lift[\rew]\cf$ has worst reward $\rew$.
Consider any $f_1$ and $f_2$ in $\gambles\group{\ps\times\rewardsw[\rew]}$ such that $f_1<f_2$.
Let
\begin{equation*}
\lambda\coloneqq\max_{\rval\in\ps}\sum_{\rewtoo\in\rewardsw[\rew]}\group{f_2\group{\rval,\rewtoo}-f_1\group{\rval,\rewtoo}}>0.
\end{equation*}
Then clearly $\frac{1}{\lambda}\group{f_2-f_1}=\poh[\rew]\hl$ for some $\hl$ in $\hls\group{\ps,\rewards}$.
Also, $\hl\neq\hl[\rew]$ because $f_1\neq f_2$.
Using the assumption that $\lift[\rew]\cf$ has worst reward $\rew$, we find that then $\hl[\rew]\in\lift[\rew]\rf\group{\set{\hl[\rew],\hl}}$.
As a consequence, by Equation~\eqref{eq:lift of poh}, we find that $0=\poh[\rew]\hl[\rew]\in\rf\group{\set{0,\poh[\rew]\hl}}=\rf\group{0,\frac{1}{\lambda}\group{f_2-f_1}}$.
Using Axiom~\ref{coh cf 4a: scaling} we infer that $0\in\rf\group{\set{0,f_2-f_1}}$, and using Axiom~\ref{coh cf 4b: independence} that indeed $f_1\in\rf\group{\set{f_1,f_2}}$.

For the direct implication of~\ref{it:equivalence coherence 3}, assume that $\cf$ satisfies Axiom~\ref{coh cf 3a: alpha}.
Consider any $\os^*$, $\os[1]^*$ and $\os[2]^*$ in $\cfdom\group{\hls\group{\ps,\rewards}}$ and assume that $\os[1]^*\subseteq\lift[\rew]\rf\group{\os[2^*]}$ and $\os[2]^*\subseteq\os^*$.
Then $\poh[\rew]\os[1]^*\subseteq\rf\group{\poh[\rew]\os[2]^*}$ by Equation~\eqref{eq:lift of poh}, and $\poh[\rew]\os[1]^*\subseteq\poh[\rew]\os^*$.
Use version~\eqref{eq:senalpha} of Axiom~\ref{coh cf 3a: alpha} to infer that then $\poh[\rew]\os[1]^*\subseteq\rf\group{\poh[\rew]\os^*}$, whence indeed $\os[1]^*\subseteq\lift[\rew]\rf\group{\os^*}$ by Equation~\eqref{eq:lift of poh}.

For the converse implication, assume that $\lift[\rew]\cf$ satisfies Axiom~\ref{coh cf 3a Seidenfeld}.
Consider any $\os$, $\os[1]$ and $\os[2]$ in $\gambles\group{\ps\times\rewardsw[\rew]}$ and assume that $\os[1]\subseteq\rf\group{\os[2]}$ and $\os[2]\subseteq\os$.
Use Lemma~\ref{lem:rescaling} to find $\lambda$ in $\posreals$ and $g$ in $\gambleson[{\ps\times\rewardsw[\rew]}]$ such that $\frac{1}{\lambda}(A+\set{g})=\poh[\rew]\os^*$ for some $\os^*$ in $\cfdom\group{\hls\group{\ps,\rewards}}$.
Analogously, we find that $\frac{1}{\lambda}\group{\os[2]+\set{g}}=\poh[\rew]\group{\os[2]^*}$ for some $\os[2]^*\subseteq\os^*$.
$\os[1]\subseteq\rf\group{\os[2]}$ implies $\os[1]\subseteq\os[2]$, so also $\frac{1}{\lambda}\group{\os[1]+\set{g}}=\poh[\rew]\group{\os[1]^*}$ for some $\os[1]^*\subseteq\os[2]^*$.
Using Axioms~\ref{coh cf 4a: scaling} and~\ref{coh cf 4b: independence}, we infer from the assumptions that $\frac{1}{\lambda}\group{\os[1]+\set{g}}\subseteq\rf\group[\big]{\frac{1}{\lambda}\group{\os[2]+\set{g}}}$, or in other words, $\poh[\rew]\os[1]^*\subseteq\rf\group{\poh[\rew]\os[2]^*}$.
Equation~\eqref{eq:lift of poh} then yields that $\os[1]^*\subseteq\lift[\rew]\rf\group{\os[2]^*}$.
As a result, using Axiom~\ref{coh cf 3a Seidenfeld}, $\os[1]^*\subseteq\lift[\rew]\rf\group{\os^*}$, which, again applying Equation~\eqref{eq:lift of poh}, results in $\frac{1}{\lambda}\group{\os[1]+\set{g}}=\poh[\rew]\os[1]^*\subseteq\rf\group{\poh[\rew]\os^*}=\rf\group[\big]{\frac{1}{\lambda}\group{\os+\set{g}}}$, and as a consequence, by Axioms~\ref{coh cf 4a: scaling} and~\ref{coh cf 4b: independence}, we find eventually that indeed $\os[1]\subseteq\rf\group{\os}$.

For the direct implication of~\ref{it:equivalence coherence 4}, assume that $\cf$ satisfies Axiom~\ref{coh cf 3b: aizerman}.
Consider any $\os^*$, $\os[1]^*$ and $\os[2]^*$ in $\cfdom\group{\hls\group{\ps,\rewards}}$ and assume that $\os[1]^*\subseteq\lift[\rew]\rf\group{\os[2]^*}$ and $\os^*\subseteq\os[1]^*$.
Then $\poh[\rew]\os[1]^*\subseteq\rf\group{\poh[\rew]\os[2]^*}$ by Equation~\eqref{eq:lift of poh}, and $\poh[\rew]\os^*\subseteq\poh[\rew]\os[1]^*$.
Use version~\eqref{eq:aizerman} of Axiom~\ref{coh cf 3b: aizerman} to infer that then $\poh[\rew]\group{\os[1]^*\setminus\os^*}=\group{\poh[\rew]\os[1]^*}\setminus\group{\poh[\rew]\os}\subseteq\rf\group{\group{\poh[\rew]\os[2]^*}\setminus\group{\poh[\rew]\os^*}}=\rf\group{\poh[\rew]\group{\os[2]^*\setminus\os^*}}$, whence indeed $\os[1]^*\setminus\os^*\subseteq\lift[\rew]\rf\group{\os[2]^*\setminus\os^*}$.

For the converse implication, assume that $\lift[\rew]\cf$ satisfies Axiom~\ref{coh cf 3b Seidenfeld}.
Consider any $\os$, $\os[1]$ and $\os[2]$ in $\cfdom\group{\gambles\group{\ps\times\rewards[\rew]}}$ and assume that $\os[1]\subseteq\rf\group{\os[2]}$ and $\os\subseteq\os[1]$.
Use Lemma~\ref{lem:rescaling} to find $\lambda$ in $\posreals$ and $g$ in $\gambleson[{\ps\times\rewardsw[\rew]}]$ such that $\frac{1}{\lambda}(\os[2]+\set{g})=\poh[\rew]\os[2]^*$ for some $\os[2]^*$ in $\cfdom\group{\hls\group{\ps,\rewards}}$.
$\os[1]\subseteq\rf\group{\os[2]}$ implies $\os[1]\subseteq\os[2]$, whence $\frac{1}{\lambda}\group{\os[1]+\set{g}}=\poh[\rew]\group{\os[1]^*}$ for some $\os[1]^*\subseteq\os[2]^*$, and 
analogously, $\frac{1}{\lambda}\group{\os+\set{g}}=\poh[\rew]\group{\os^*}$ for some $\os^*\subseteq\os[1]^*$.
Using Axioms~\ref{coh cf 4a: scaling} and~\ref{coh cf 4b: independence} we find that $\frac{1}{\lambda}\group{\os[1]+\set{g}}\subseteq\rf\group[\big]{\frac{1}{\lambda}\group{\os[2]+\set{g}}}$, or in other words, $\poh[\rew]\os[1]^*\subseteq\rf\group{\poh[\rew]\os[2]^*}$.
Equation~\eqref{eq:lift of poh} then tells us that $\os[1]^*\subseteq\lift[\rew]\rf\group{\os[2]^*}$, which, using Axiom~\ref{coh cf 3b Seidenfeld}, results in $\os[1]^*\setminus\os^*\subseteq\lift[\rew]\rf\group{\os[2]^*\setminus\os^*}$.
Again applying Equation~\eqref{eq:lift of poh} results in
\begin{align*}
\frac{1}{\lambda}\group{\group{\os[1]\setminus\os}+\set{g}}
&=\frac{1}{\lambda}\group{\os[1]+\set{g}}\setminus\frac{1}{\lambda}\group{\os+\set{g}}
=\group{\poh[\rew]\os[1]^*}\setminus\group{\poh[\rew]\os^*}
=\poh[\rew]\group{\os[1]^*\setminus\os^*}\\
&\subseteq\rf\group{\poh[\rew]\group{\os[2]^*\setminus\os^*}}
=\rf\group{\group{\poh[\rew]\os[2]^*}\setminus\group{\poh[\rew]\os^*}}\\
&=\rf\group{\frac{1}{\lambda}\group{\os[2]+\set{g}}\setminus\frac{1}{\lambda}\group{\os+\set{g}}}
=\rf\group{\frac{1}{\lambda}\group{\group{\os[2]\setminus\os}+\set{g}}},
\end{align*}
and as a consequence, by Axioms~\ref{coh cf 4a: scaling} and~\ref{coh cf 4b: independence}, we find eventually that indeed $\os[1]\setminus\os\subseteq\rf\group{\os[2]\setminus\os}$.

For~\ref{it:equivalence coherence 5}, consider any $\os[1]^*$ and $\os[2]^*$ in $\cfdom\group{\hls\group{\ps,\rewards}}$, any $\hl$ in $\hls\group{\ps,\rewards}$, and any $\alpha$ in $(0,1]$.
Consider the following chain of equivalences
\begin{align*}
&\os[1]^*\crel[{\lift[\rew]\cf}]\os[2]^*\\
&\quad\iff
\os[1]^*\subseteq\lift[\rew]\rf\group{\os[1]^*\cup\os[2]^*}&&\text{by Equation~\eqref{eq:choice relation}}\\
&\quad\iff
\poh[\rew]\os[1]^*\subseteq\rf\group{\poh[\rew]\group{\os[1]^*\cup\os[2]^*}}&&\text{by Equation~\eqref{eq:lift of poh}}\\
&\quad\iff
\poh[\rew]\alpha\os[1]^*\subseteq\rf\group{\poh[\rew]\alpha\group{\os[1]^*\cup\os[2]^*}}&&\text{by Axiom~\ref{coh cf 4a: scaling}}\\
&\quad\iff
\poh[\rew]\group{\alpha\os[1]^*+(1-\alpha)\set{\hl}}\subseteq\rf\group{\poh[\rew]\group{\alpha\group{\os[1]^*\cup\os[2]^*}+(1-\alpha)\set{\hl}}}&&\text{by Axiom~\ref{coh cf 4b: independence}}\\
&\quad\iff
\alpha\os[1]^*+(1-\alpha)\set{\hl}\subseteq\lift[\rew]\rf\group{\alpha\group{\os[1]^*\cup\os[2]^*}+(1-\alpha)\set{\hl}}&&\text{by Equation~\eqref{eq:lift of poh}}\\
&\quad\iff
\group{\alpha\os[1]^*+(1-\alpha)\set{\hl}}\crel[{\lift[\rew]\cf}]\group{\alpha\os[2]^*+(1-\alpha)\set{\hl}}&&\text{by Equation~\eqref{eq:choice relation}},
\end{align*}
which tells us that $\lift[\rew]\cf$ satisfies Axiom~\ref{coh cf 4 Seidenfeld}.

For the direct implication of~\ref{it:equivalence coherence 6}, assume that $\cf$ satisfies Axiom~\ref{coh cf 5: convexity}.
Consider any $\os^*$ and $\os[1]^*$ in $\cfdom\group{\hls\group{\ps,\rewards}}$ and assume that $\os^*\subseteq\os[1]^*\subseteq\ch\group{\os^*}$.
Then $\poh[\rew]\os^*\subseteq\poh[\rew]\os[1]^*\subseteq\ch\group{\poh[\rew]\os^*}$, whence $\cf\group{\poh[\rew]\os^*}\subseteq\cf\group{\poh[\rew]\os[1]^*}$ by Axiom~\ref{coh cf 5: convexity}.
Use Equation~\eqref{eq:lift of poh} to infer that then indeed $\lift[\rew]\cf\group{\os^*}\subseteq\lift[\rew]\cf\group{\os[1]^*}$.

For the converse implication, assume that $\lift[\rew]\cf$ satisfies Axiom~\ref{coh cf 5 Seidenfeld}.
Consider any $\os$ and $\os[1]$ in $\cfdom\group{\gambles\group{\ps\times\rewards[\rew]}}$ and assume that $\os\subseteq\os[1]\subseteq\ch\group{\os}$.
Use Lemma~\ref{lem:rescaling} to find $\lambda$ in $\posreals$ and $g$ in $\gambleson[{\ps\times\rewardsw[\rew]}]$ such that $\frac{1}{\lambda}(\os[1]+\set{g})=\poh[\rew]\os[1]^*$ for some $\os[1]^*$ in $\cfdom\group{\hls\group{\ps,\rewards}}$
, and analogously, $\frac{1}{\lambda}\group{\os+\set{g}}=\poh[\rew]\group{\os^*}$ for some $\os^*\subseteq\os[1]^*$.
From $\os[1]\subseteq\ch\group{\os}$ infer that $\frac{1}{\lambda}\group{\os[1]+\set{g}}\subseteq\ch\group{\frac{1}{\lambda}\group{\os+\set{g}}}$, or in other words, $\poh[\rew]\os[1]^*\subseteq\ch\group{\poh[\rew]\os^*}$.
Then we claim that $\os[1]^*\subseteq\ch\group{\os^*}$.
To prove this, consider any $\hl$ in $\os[1]^*$.
Then there are $n$ in $\nats$, $\hl[i]$ in $\os$, and $\alpha_i\geq0$ such that $\sum_{i=1}^n\alpha_i=1$ and $\hl\group{\cdot,\rewtoo}=\sum_{i=1}^n\alpha_i\hl[i]\group{\cdot,\rewtoo}$ for all $\rewtoo$ in $\rewardsw[\rew]$.
Moreover,
\begin{align*}
\hl\group{\cdot,\rew}
&=1-\sum_{\rewtoo\in\rewardsw[\rew]}\hl\group{\cdot,\rewtoo}
=1-\sum_{\rew\in\rewardsw[\rew]}\sum_{i=1}^n\alpha_i\hl[i]\group{\cdot,\rewtoo}\\
&=\sum_{i=1}^n\alpha_i-\sum_{i=1}^n\alpha_i\sum_{\rew\in\rewardsw[\rew]}\hl[i]\group{\cdot,\rewtoo}
=\sum_{i=1}^n\alpha_i\group[\big]{1-\sum_{\rew\in\rewardsw[\rew]}\hl[i]\group{\cdot,\rewtoo}}
=\sum_{i=1}^n\alpha_i\hl[i]\group{\cdot,\rew},
\end{align*}
so indeed $\hl\in\ch\group{\os^*}$.
Use Axiom~\ref{coh cf 5 Seidenfeld} to infer that then $\lift[\rew]\cf\group{\os^*}\subseteq\lift[\rew]\cf\group{\os[1]^*}$.
Equation~\eqref{eq:lift of poh} turns this into $\cf\group[\big]{\frac{1}{\lambda}\group{\os+\set{g}}}=\cf\group{\poh[\rew]\os^*}\subseteq\cf\group{\poh[\rew]\os[1]^*}=\cf\group[\big]{\frac{1}{\lambda}\group{\os[1]+\set{g}}}$, which by Axioms~\ref{coh cf 4a: scaling} and~\ref{coh cf 4b: independence}, results in $\cf\group{\os}\subseteq\cf\group{\os[1]}$.
\end{proof}

We conclude that our discussion of choice functions on linear spaces subsumes the treatment of choice functions on horse lotteries satisfying Axiom~\ref{coh cf 4 Seidenfeld}.
Using the connections established above, all the results that we will prove later on are also applicable to choice functions on horse lotteries that satisfy the corresponding rationality axioms.

\section{The link with desirability}\label{sec:link with desirability}
\citet{vancamp2017} have studied in some detail how the coherent choice functions in the sense of Definition~\ref{def: rationality axioms for choice functions} can be related to coherent sets of desirable options (gambles).
As an example, given a coherent set of desirable options $\sodv$, the choice function that identifies the undominated---under the preference relation induced by $\sodv$---options, is coherent.
This choice rule is called \emph{maximality} (see Equation~\eqref{eq:CD} further on).
There are other rules that induce coherent choice functions, such as E-admissibility---those choice functions identify the options whose (precise) expectation is maximal for at least one probability mass function in the credal set induced by $\sodv$.
Since we have shown in earlier work~\citep[Proposition~13]{vancamp2017} that maximality leads to the most conservative coherent choice function that reflects the binary choices represented by $\sodv$ \citep[see also][Theorem~3]{Bradley2015}, we focus on maximality as the connection between desirability and choice functions.
Here, we investigate what remains of this connection when we require in addition that our choice functions should satisfy Axiom~\ref{coh cf 5: convexity}.

We recall that a set of desirable options is simply a subset of the vector space $\vs$.
The underlying idea is that a subject strictly prefers each option in this set to the status quo~$0$.
As for choice functions, we pay special attention to \emph{coherent} sets of desirable options.

\begin{definition}\label{def: rationality axioms for sets of desirable options}
A set of desirable options $\sodv$ is called \emph{coherent} if for all $\vect$ and $\vect[v]$ in $\vs$, and all $\lambda$ in $\posreals$:
\begin{enumerate}[label=\upshape D$_{\arabic*}$.,ref=\upshape D$_{\arabic*}$,leftmargin=*,noitemsep,topsep=0pt]
\item\label{coh sodv 1: 0 is not desirable}
  $0\notin\sodv$;
\item\label{coh sodv 2: positive vectors are desirable}
  $\vspos\subseteq\sodv$;
\item\label{coh sodv 3: positive scaling is desirable}
  if $\vect\in\sodv$ then $\lambda\vect\in\sodv$;
\item\label{coh sodv 4: sum is desirable}
  if $\vect,\vect[v]\in\sodv$ then $\vect+\vect[v]\in\sodv$.
\end{enumerate}
We collect all coherent sets of desirable options in the set~$\allcohsodvs$.
\end{definition}
\noindent More details can be found in a number of papers and books~\citep{walley1991,walley2000,moral2005b,miranda2010,couso2011,cooman2010,cooman2011b,quaeghebeur2012:itip,quaeghebeur2012:statements,debock2015}.

Axioms~\ref{coh sodv 3: positive scaling is desirable} and~\ref{coh sodv 4: sum is desirable} guarantee that a coherent $\sodv$ is a convex cone.
This convex cone induces a strict partial order $\prel[\sodv]$ on~$\vs$, by letting
\begin{equation}\label{eq:order-desirs}
\vect\prel[\sodv]\vect[v]\iff0\prel[\sodv]\vect[v]-\vect\iff\vect[v]-\vect\in\sodv,
\end{equation}
so $\sodv=\cset{\vect\in\vs}{0\prel[\sodv]\vect}$ \citep{cooman2010,quaeghebeur2012:itip}.
$\sodv$ and $\prel[\sodv]$ are mathematically equivalent: given one of $\sodv$ or $\prel[\sodv]$, we can determine the other unequivocally using the formulas above.
When it is clear from the context which set of desirable options $\sodv$ we are working with, we often refrain from mentioning the explicit reference to $\sodv$ in $\prel[\sodv]$ and then we simply write $\prel$.
Coherence for sets of desirable options transfers to binary relations $\prel$ as follows: $\prel$ must be a strict partial order---meaning that it is irreflexive and transitive---such that ${\svo}\subseteq{\prel}$, and must satisfy the two characteristic properties of Equations~\eqref{eq: vector ordering 1: add constant vector to both sides} and~\eqref{eq: vector ordering 2: multiply with positive lambda}.

What is the relationship between choice functions and sets of desirable options?
Since we have just seen that sets of desirable options represent binary preferences, we see that we can associate a set of desirable options $\sodv[\cf]$ with every given choice function $\cf$ by focusing on its binary choices:
\begin{equation}\label{eq:binary:behaviour}
\vect\prel[{\sodv[\cf]}]\vect[v] 
\iff\vect[v]-\vect[u]\in\sodv[\cf]
\iff\vect[u]\in\rf\group{\set{\vect,\vect[v]}}
\text{ for all $\vect,\vect[v]$ in $\vs$}.
\end{equation}
$\sodv[\cf]$ is a coherent set of desirable options if $\cf$ is a coherent choice function \citep[Proposition~12]{vancamp2017}.
Conversely \citep[Proposition~13]{vancamp2017}, if we start out with a coherent set of desirable options $\sodv$ then the set $\cset{\cf\in\allcohcfs}{\sodv[\cf]=\sodv}$ of all coherent choice functions whose binary choices are represented by $\sodv$, is non-empty, and its smallest, or least informative, element $\cf[\sodv]\coloneqq\inf\cset{\cf\in\allcohcfs}{\sodv[\cf]=\sodv}$ is given by:
\begin{equation}\label{eq:CD}
\cf[\sodv]\group{\os}
\coloneqq\cset{\vect\in\os}{\group{\forall\vect[v]\in\os}\vect[v]-\vect\notin\sodv}
=\cset{\vect\in\os}{\group{\forall\vect[v]\in\os}\vect\nprel\vect[v]}
\text{ for all $\os$ in $\cfdom$.}
\end{equation}
It selects all options from $\os$ that are undominated, or maximal, under the ordering $\prel[\sodv]$, or in other words, it is the corresponding choice function based on Walley--Sen maximality.
This $\cf[\sodv]$ is easy to characterise:

\begin{proposition}\label{prop: CD in terms of intersection}
Given any coherent set of desirable options $\sodv$, then
\begin{equation*}
0\in\cf[\sodv]\group{\set{0}\cup\os}
\iff
\sodv\cap\os=\emptyset
\text{ for all $\os$ in $\cfdom$.}
\end{equation*}
\end{proposition}

\begin{proof}
By Equation~\eqref{eq:CD}, $0\in\cf[\sodv]\group{\set{0}\cup\os}\iff\group{\forall\vect[v]\in\set{0}\cup\os}\vect[v]\notin\sodv\iff\group{\set{0}\cup\os}\cap\sodv=\emptyset$, which is equivalent to $\os\cap\sodv=\emptyset$, because $0\notin\sodv$ for any coherent $\sodv$.
\end{proof}
\noindent
Although $\cf[\sodv]$ is coherent when $\sodv$ is, it does not necessarily satisfy the additional Axiom~\ref{coh cf 5: convexity}, as the following counterexample shows.

\begin{example}\label{ex:CD not convex}
Consider the two-dimensional vector space $\vs=\reals^2$. 
We provide it with the component-wise vector ordering $\vo$, and consider the vacuous set of desirable options $\sodv=\cset{\vect\in\vs}{\vect\svoi0}=\vspos$, which is coherent.
By Proposition~\ref{prop: CD in terms of intersection}, $0\in\cf[\sodv]\group{\set{0}\cup\os}\iff\os\cap\vspos=\emptyset$ for all $\os$ in $\cfdom$. 
To show that $\cf[\sodv]$ does not satisfy Axiom~\ref{coh cf 5: convexity}, consider $\os=\set{0,\vect,\vect[v]}$, where $\vect\coloneqq(-1,2)$ and $\vect[v]\coloneqq(2,-1)$.
We find that $0\in\cf[\sodv]\group{\os}$ because $\set{\vect,\vect[v]}\cap\vspos=\emptyset$, since $\vect\nsvoi0$ and $\vect[v]\nsvoi0$.

However, for the option set $\os[1]=\os\cup\set{\frac{\vect+\vect[v]}{2}}\subseteq\ch\group{\os}$, we find that $\frac{\vect+\vect[v]}{2}=(\nicefrac{1}{2},\nicefrac{1}{2})\svoi0$ and therefore $0\notin\cf[\sodv]\group{\os[1]}$, meaning that Axiom~\ref{coh cf 5: convexity} is not satisfied.
\hfill$\lozenge$
\end{example}
\noindent For the specific coherent set of desirable options $\sodv$ considered in Example~\ref{ex:CD not convex}, the corresponding choice function $\cf[\sodv]$ fails to satisfy \ref{coh cf 5: convexity}. 
However, there are other sets of desirable options $\sodv$ for which $\cf[\sodv]$ does satisfy the convexity axiom. 
They are identified in the next proposition.

\begin{proposition}\label{prop:CD coherent iff D lexicographic} 
Consider any coherent set of desirable options $\sodv$, then the corresponding coherent choice function $\cf[\sodv]$ satisfies Axiom~\ref{coh cf 5: convexity} if and only if $\sodv^c$ is a convex cone, or in other words, if and only if\/ $\Posi\group{\sodv^c}=\sodv^c$, or equivalently, $\Posi\group{\sodv^c}\cap\sodv=\emptyset$.
\end{proposition}

\begin{proof}
\citet[Proposition~13]{vancamp2017} have already shown that $\cf[\sodv]$ is a coherent choice function.

For necessity, assume that $\Posi\group{\sodv^c}\neq\sodv^c$, or equivalently, that $\Posi\group{\sodv^c}\cap\sodv\neq\emptyset$.
Then there is some option $\vect$ in $\sodv$ such that $\vect\in\Posi\group{\sodv^c}$, meaning that there are $n$ in $\nats$, $\lambda_k$ in $\posreals$ and $\vect[u_k]$ in $\sodv^c$ such that $\vect=\sum_{k=1}^n\lambda_k\vect[u_k]$.
Let $\os\coloneqq\set{0,\vect[u_1],\dots,\vect[u_n]}$ and $\os[1]\coloneqq\os\cup\set{\vect}$.
Due to the coherence of $\sodv$ [more precisely Axiom~\ref{coh sodv 3: positive scaling is desirable}], we can rescale $\vect\in\sodv$ while keeping the $\vect[u_k]$ fixed, in such a way that we achieve that $\sum_{k=1}^n\lambda_k=1$, whence $\os\subseteq\os[1]\subseteq\ch\group{\os}$.
We find that $0\in\cf[\sodv]\group{\os}$ by Proposition~\ref{prop: CD in terms of intersection}, because $\os\cap\sodv=\emptyset$, but $0\notin\cf[\sodv]\group{\os[1]}$ because $\vect\in\sodv$, so $\os[1]\cap\sodv\neq\emptyset$.
This tells us that $\cf[\sodv]$ does not satisfy Axiom~\ref{coh cf 5: convexity}, because clearly $\cf[\sodv]\group{\os}\nsubseteq\cf[\sodv]\group{\os[1]}$.

For sufficiency, assume that $\cf[\sodv]$ does not satisfy Axiom~\ref{coh cf 5: convexity}.
Consider any $\os$ and $\os[1]$ in $\cfdom$ for which $\os\subseteq\os[1]\subseteq\ch\group{\os}$.
Then there is some $\vect$ in $\os$ such that $\vect\in\cf[\sodv]\group{\os}$ and $\vect\notin\cf[\sodv]\group{\os[1]}$.
Due to Axiom~\ref{coh cf 4b: independence}, we find that $0\in\cf[\sodv](\os-\set{\vect})$ and $0\notin\cf[\sodv]\group{\os[1]-\set{\vect}}$, or equivalently, by Proposition~\ref{prop: CD in terms of intersection}, that $\os-\set{\vect}\subseteq\sodv^c$ and $\os[1]-\set{\vect}\cap\sodv\neq\emptyset$.
But $\os[1]-\set{\vect}\subseteq\ch\group{\os}-\set{\vect}=\ch\group{\os-\set{\vect}}\subseteq\Posi\group{\os-\set{\vect}}\subseteq\Posi\group{\sodv^c}$, so $\Posi\group{\sodv^c}\cap\sodv\neq\emptyset$.
\end{proof}
\noindent
This proposition seems to indicate that there is something special about coherent sets of desirable options whose complement is a convex cone too.
We give them a special name that will be motivated and explained in the next section.

\begin{definition}
A coherent set of desirable options $\sodv$ is called \emph{lexicographic} if
\begin{equation*}
\Posi(\sodv^c)=\sodv^c
\text{, or, equivalently, if }
\Posi(\sodv^c)\cap\sodv=\emptyset.
\end{equation*}
We collect all the lexicographic coherent sets of desirable options in~$\cohlexisodvs$.
\end{definition}

Another important subclass $\maxcohsodvs$ of coherent sets of desirable options collects all the maximally informative, or \emph{maximal}, ones:
\begin{equation*}
\maxcohsodvs
\coloneqq\cset{\sodv\in\allcohsodvs}
{\group{\forall\sodv'\in\allcohsodvs}\sodv\subseteq\sodv'\then\sodv=\sodv'}.
\end{equation*}
The sets of desirable options in $\maxcohsodvs$ are the undominated elements of the complete infimum-semilattice $(\allcohsodvs,\subseteq)$.
\citet{couso2011} have proved the following elegant and useful characterisation of these maximal elements:

\begin{proposition}\label{prop: maximal sets of desirable options}
Given any coherent set of desirable options $\sodv$ and any non-zero option $\vect\notin\sodv$, then $\Posi\group{\sodv\cup\set{-\vect}}$ is a coherent set of desirable options.
As a consequence, a coherent set of desirable options $\sodv$ is maximal if and only if
\begin{equation*}
\group{\forall\vect\in\vs\setminus\set{0}}\group{\vect\in\sodv\text{ or }-\vect\in\sodv}.
\end{equation*}
\end{proposition}
\noindent
\Citet{cooman2010} have proved that the set of all coherent sets of desirable options is \emph{dually atomic}, meaning that that any coherent set of desirable options $\sodv$ is the infimum of its non-empty set of dominating maximal coherent sets of desirable options:

\begin{proposition}\label{prop: sodvs are dually atomic}
For any coherent set of desirable options~$\sodv$, its set of dominating maximal coherent sets of desirable options $\maxcohsodvsdom{\sodv}\coloneq\cset{\hat{\sodv}\in\maxcohsodvs}{\sodv\subseteq\hat{\sodv}}$ is non-empty, and $\sodv=\bigcap\maxcohsodvsdom{\sodv}$.
\end{proposition}
\noindent
Any maximal coherent set of desirable options is also a lexicographic one: $\maxcohsodvs\subseteq\cohlexisodvs$.
To see this, consider a maximal $\sodv$ and arbitrary $n$ in $\nats$, $\vect[u_k]$ in $\sodv^c$ and $\lambda_k\in\posreals$ for $k\in\set{1,\dots,n}$.
Then since all $-\vect[u_k]\in\sodv\cup\set{0}$ by Proposition~\ref{prop: maximal sets of desirable options}, we infer that $-\sum_{k=1}^n\lambda_k\vect[u_k]\in\sodv\cup\set{0}$, because the coherent $\sodv$ is in particular a convex cone.
If $\sum_{k=1}^n\lambda_k\vect[u_k]=0$, then $\sum_{k=1}^n\lambda_k\vect[u_k]\in\sodv^c$ by Axiom~\ref{coh sodv 1: 0 is not desirable}.
If $\sum_{k=1}^n\lambda_k\vect[u_k]\neq0$, then $-\sum_{k=1}^n\lambda_k\vect[u_k]\in\sodv$, and coherence then guarantees that, here too, $\sum_{k=1}^n\lambda_k\vect[u_k]\in\sodv^c$. 
We conclude that $\sodv^c$ is indeed a convex cone.

\section{Lexicographic choice functions}\label{sec:lexicographic}
In this section, we embark on a more detailed study of lexicographic sets of desirable options, and amongst other things, explain where their name comes from.
We will restrict ourselves here to the special case where $\vs$ is the linear space $\gambleson$ of all gambles on a \emph{finite} possibility space $\ps$, provided with the component-wise order $\leq$ as its vector ordering.

We first show that the lower expectation functional associated with a lexicographic $\sodv$ is actually a linear prevision \Citep{walley1991,troffaes2013:lp}.

\begin{proposition}\label{prop:lexicographic D induces linear prevision}
For any $\sodv$ in $\cohlexisodvs$, the coherent lower prevision $\lp[\sodv]$ on $\gambleson$ defined by
\begin{equation*}
\lp[\sodv]\group{f}
\coloneqq\sup\cset{\mu\in\reals}{f-\mu\in\sodv}
\text{ for all $f$ in $\gambleson$}
\end{equation*}
is a \emph{linear prevision}: a real linear functional that is positive---so $\group{\forall f\geq0}\lp[\sodv](f)\geq0$---and normalised---meaning that $\lp[\sodv](1)=1$.
\end{proposition}

\begin{proof}
Consider any $f$ in $\gambles$ and $\epsilon$ in $\posreals$, then we first prove that $f\in\sodv$ or $\epsilon-f\in\sodv$. 
Assume \emph{ex absurdo} that $f\notin\sodv$ and $\epsilon-f\notin\sodv$. 
Then, because by assumption $\Posi(\sodv^c)=\sodv^c$ is a convex cone, we also have that $f+\epsilon-f=\epsilon\notin\sodv$, which contradicts Axiom~\ref{coh sodv 2: positive vectors are desirable}.
Now, Proposition~6 by~\citet{miranda2010} guarantees that for any such $\sodv$, the corresponding functional $\lp[\sodv]$ is indeed a linear prevision.
\end{proof}

To get some feeling for what these lexicographic models represent, we first look at the special case of binary possibility spaces $\set{a,b}$, leading to a two-dimensional option space $\vs=\gambleson[\set{a,b}]$ provided with the point-wise order. 
It turns out that lexicographic sets of desirable options (gambles) are easy to characterise there, so we have a simple expression for $\cohlexisodvs$.

\begin{proposition}\label{prop:lexic-binary}
All lexicographic coherent sets of desirable gambles on the binary possibility space $\set{a,b}$ are given by (see also Figure~\ref{fig:lexicopgraphic:two:dimensions}):
\begin{equation*}
\cohlexisodvs
\coloneq
\cset{\sodv[\rho],\sodv[\rho]^a,\sodv[\rho]^b}{\rho\in(0,1)}
\cup
\set{\sodv[0],\sodv[1]}
=\cset{\sodv[\rho]}{\rho\in(0,1)}
\cup
\maxcohsodvs,
\end{equation*}
where
\begin{align*}
&\sodv[\rho]
\coloneq
\cset{\lambda\group{\rho-\indset{a}}}{\lambda\in\reals}+\vspos
=\Span\group{\set{\rho-\indset{a}}}+\vspos
&&\text{for all $\rho$ in $(0,1)$}\\
&\sodv[\rho]^a
\coloneq
\sodv[\rho]\cup\cset{\lambda\group{\rho-\indset{a}}}{\lambda\in\negreals}
=\sodv[\rho]\cup\Posi\group{\set{\indset{a}-\rho}}
&&\text{for all $\rho$ in $(0,1)$}\\
&\sodv[\rho]^b
\coloneq
\sodv[\rho]\cup\cset{\lambda\group{\rho-\indset{a}}}{\lambda\in\posreals}
=\sodv[\rho]\cup\Posi\group{\set{\rho-\indset{a}}}
&&\text{for all $\rho$ in $(0,1)$.}\\
&\sodv[0]\coloneq\set{f\in\vs: f(a)>0}\cup\vspos \\
&\sodv[1]\coloneq\set{f\in\vs: f(b)>0}\cup\vspos.
\end{align*}
\end{proposition}

\begin{proof}
We first observe that every set of desirable options in $\cohlexisodvs$ is coherent.
Indeed, for any $\rho$ in $(0,1)$, $\sodv[\rho]$ is the smallest coherent set of desirable gambles corresponding to the linear prevision $\E$, with $\mf\coloneq(\rho,1-\rho)$, while $\sodv[\rho]^a$, $\sodv[\rho]^b$ are maximal coherent sets of desirable gambles corresponding to the same linear prevision $\E$.
Finally, $\sodv[0]$ is the maximal (and only) coherent set of desirable gambles corresponding to $\E$ with $\mf\coloneq(0,1)$, while $\sodv[1]$ is the maximal (and only) coherent set of desirable gambles corresponding to $\E$ with $\mf\coloneq(1,0)$.

We now prove that we recover all lexicographic coherent sets of desirable gambles in this way.
Consider any lexicographic coherent set of desirable gambles $\sodv^*$.
Then $\lp[\sodv^*]$ is a linear prevision, by Proposition~\ref{prop:lexicographic D induces linear prevision}, so $\lp[\sodv^*]$ is characterised (i) by the mass function $(1,0)$, (ii) by the mass function $(0,1)$, or (iii) by the mass function $(\rho^*,1-\rho^*)$ for some $\rho^*$ in $(0,1)$.
If (i), the only coherent set of desirable gambles that induces the linear prevision with mass function $(1,0)$ is $\sodv[1]\in\cohlexisodvs$.
If (ii), the only coherent set of desirable gambles that induces the linear prevision with mass function $(0,1)$ is $\sodv[0]\in\cohlexisodvs$.
If (iii), there are only three coherent sets of desirable gambles that induce the linear prevision with mass function $(\rho^*,1-\rho^*)$: $\sodv[\rho^*]$, $\sodv[\rho^*]^a$ and $\sodv[\rho^*]^b$, and all are elements of $\cohlexisodvs$.
\end{proof}
\noindent
In the language of sets of desirable gambles \citep[see for instance][]{quaeghebeur2012:itip}, this means that in the binary case lexicographic sets of desirable gambles are either \emph{maximal} or \emph{strictly desirable} with respect to a linear prevision.

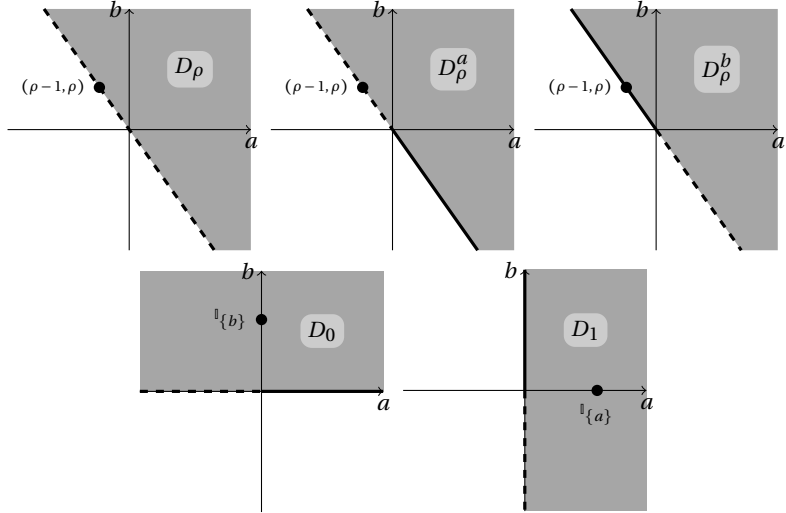
\begin{figure}[ht]
\centering
\begin{tikzpicture}[scale=0.8]\small
\draw[fill=gray,draw=none,opacity=0.7] (-1.4,2) -- (2,2) -- (2,-2) -- (1.4,-2) -- cycle;
\draw[->] (-2,0) -- (2,0) node[below] {$a$};
\draw[->] (0,-2) -- (0,2) node[left] {$b$};
\coordinate[circle,inner sep=1.5pt,fill=black,label={left:\tiny$(\rho-1,\rho)$}] (l) at (-0.49,0.7);
\draw[very thick,dashed] (-1.4,2) -- (1.4,-2);
\node[rounded corners=0.15cm,fill=gray!40!white] at (1,1) {$\sodv[\rho]$};
\end{tikzpicture}
\begin{tikzpicture}[scale=0.8]\small
\draw[fill=gray,draw=none,opacity=0.7] (-1.4,2) -- (2,2) -- (2,-2) -- (1.4,-2) -- cycle;
\draw[->] (-2,0) -- (2,0) node[below] {$a$};
\draw[->] (0,-2) -- (0,2) node[left] {$b$};
\coordinate[circle,inner sep=1.5pt,fill=black,label={left:\tiny$(\rho-1,\rho)$}] (l) at (-0.49,0.7);
\draw[very thick,dashed] (-1.4,2) -- (0,0);
\draw[very thick] (0,0) -- (1.4,-2);
\node[rounded corners=0.15cm,fill=gray!40!white] at (1,1) {$\sodv[\rho]^a$};
\end{tikzpicture}
\begin{tikzpicture}[scale=0.8]\small
\draw[fill=gray,draw=none,opacity=0.7] (-1.4,2) -- (2,2) -- (2,-2) -- (1.4,-2) -- cycle;
\draw[->] (-2,0) -- (2,0) node[below] {$a$};
\draw[->] (0,-2) -- (0,2) node[left] {$b$};
\coordinate[circle,inner sep=1.5pt,fill=black,label={left:\tiny$(\rho-1,\rho)$}] (l) at (-0.49,0.7);
\draw[very thick] (-1.4,2) -- (0,0);
\draw[very thick,dashed] (0,0) -- (1.4,-2);
\node[rounded corners=0.15cm,fill=gray!40!white] at (1,1) {$\sodv[\rho]^b$};
\end{tikzpicture}
\begin{tikzpicture}[scale=0.8]\small
\draw[fill=gray,draw=none,opacity=0.7] (-2,0) -- (-2,2) -- (2,2) -- (2,0) -- cycle;
\draw[->] (-2,0) -- (2,0) node[below] {$a$};
\draw[->] (0,-2) -- (0,2) node[left] {$b$};
\coordinate[circle,inner sep=1.5pt,fill=black,label={left:\tiny$\ind[\set{b}]$}] (l) at (0,1.19);
\draw[very thick,dashed] (-2,0) -- (0,0);
\draw[very thick] (0,0) -- (2,0);
\node[rounded corners=0.15cm,fill=gray!40!white] at (1,1) {$\sodv[0]$};
\end{tikzpicture}
\begin{tikzpicture}[scale=0.8]\small
\draw[fill=gray,draw=none,opacity=0.7] (0,2) -- (2,2) -- (2,-2) -- (0,-2) -- cycle;
\draw[->] (-2,0) -- (2,0) node[below] {$a$};
\draw[->] (0,-2) -- (0,2) node[left] {$b$};
\coordinate[circle,inner sep=1.5pt,fill=black,label={below:\tiny$\ind[\set{a}]$}] (l) at (1.19,0);
\draw[very thick] (0,0) -- (0,2);
\draw[very thick,dashed] (0,0) -- (0,-2);
\node[rounded corners=0.15cm,fill=gray!40!white] at (1,1) {$\sodv[1]$};
\end{tikzpicture}
\caption{The lexicographic coherent sets of desirable gambles on the binary possibility space $\set{a,b}$, with $\rho\in\group{0,1}$.}
\label{fig:lexicopgraphic:two:dimensions}
\end{figure}

We now turn to the more general finite-dimensional case.
Recall that a \emph{lexicographic order} $\lo$ with $\ell\in\nats$ \emph{layers} on a vector space $\vs$ of finite dimension $n$ is defined by
\begin{equation}\label{eq:lo-def}
\vect\lo\vect[v]
\iff
\group{\exists k\in\set{1,\dots,\ell}}
\group{\vect_k<\vect[v]_k
\text{ and }
\group{\forall j\in\set{1,\dots,k-1}}
\vect_j=\vect[v]_j},
\end{equation}
and denote, as usual, its reflexive version $\rlo$ as $\vect\rlo\vect[v]\iff\group{\vect\lo\vect[v]\text{ or }\vect=\vect[v]}$ for any two vectors $\vect=(\vect_1,\dots,\vect_n)$ and $\vect[v]=(\vect[v]_1,\dots,\vect[v]_n)$ in $\vs$.
A \emph{lexicographic probability system} is an $\ell$-tuple $\mf\coloneqq(\mf[1],\dots,\mf[\ell])$ of probability mass functions on a possibility space $\ps$ of cardinality $n$.
We associate with this tuple $\mf$ an expectation operator $\E[\mf]\coloneqq(\E[{\mf[1]}],\dots,\E[{\mf[\ell]}])$, and a (strict) preference relation~$\slo$ on~$\gambleson$, defined by:
\begin{equation}\label{eq:order-mfs}
f\slo g
\iff\E[\mf]\group{f}\lo\E[\mf]\group{g},
\text{ for all $f,g$ in $\gambleson$.}
\end{equation}
We refer to work by~\citet{blume1991}, \citet{fishburn1982} and \citet{seidenfeld1990} for more details on generic lexicographic probability systems.
The connection between lexicographic probability systems and sets of desirable gambles has also been studied by \citet{cozman2015}, and the connection with full conditional measures by \citet{halpern2010} and \citet{hammond1994}. 
Below, we first recall a number of relevant basic properties of lexicographic orders in Propositions~\ref{prop:lexicographic is strict weak order} and~\ref{prop:vector order}.
We then provide a characterisation of lexicographic sets of desirable gambles in terms of lexicographic orders in Theorem~\ref{thm:lexico sodvs}.

Remark that the reflexive version $\rslo$ of $\slo$---defined by $f\rslo g\iff\E\group{f}\rlo\E\group{g}$ for all $f$ and $g$ in $\gambleson$---is a total order on $\gambleson$ \citep{blume1991}.

In what follows, we will restrict our attention to lexicographic probability systems $\mf$ that satisfy the following condition:
\begin{equation}\label{eq:null event lps}
\group{\forall\rval\in\ps}
\group{\exists k\in\set{1,\dots,\ell}}
\mf[k]\group{\rval}>0.
\end{equation}
This condition requires that there should be no possible outcome in $\ps$ that has zero probability in every layer.
It is closely related to the notion of a \emph{Savage-null event} \citep[Section~2.7]{savage1972}:
\begin{definition}\label{def:Savage-null events}
An event $\ev\subseteq\ps$ is called \emph{Savage-null}\/ if $\group{\forall f,g\in\gambleson}\ind f\rslo\ind g$.
The event $\emptyset$ is always Savage-null, and is called the \emph{trivial Savage-null event}.
\end{definition}

An important feature of preference relations $\slo$ based on lexicographic probability systems is the \emph{incomparability relation} $\ilo$, defined by: $f\ilo g$ if and only if $f\nslo g$ and $g\nslo f$ for all $f$ and $g$ in $\gambleson$. Since $\rslo$ is a total order, it follows that 
\begin{equation}\label{eq:indif-lexic}
f\ilo g \iff \E\group{f}=\E\group{g}.
\end{equation}
Finally, it also follows that
\begin{equation}\label{eq:non-dom-lexic}
{f}\nslo{g}
\iff
g\slo f\text{ or }g\ilo f
\iff
\E\group{g}\rlo\E\group{f}.
\end{equation}

\begin{proposition}\label{prop:Savage-null}
Consider any lexicographic probability system $\mf=(\mf[1],\dots,\mf[\ell])$.
Then Condition~\eqref{eq:null event lps} holds if and only if there are no non-trivial Savage-null events.
\end{proposition}

\begin{proof}
For the direct implication, consider any lexicographic probability system $\mf$ that satisfies Condition~\eqref{eq:null event lps}, and consider any non-empty event $\ev\subseteq\ps$.
Consider any $\rval$ in $\ev$, then $\ind\geq\indset{\rval}$ so $\E[{\mf[k]}]\group{\ind}\geq\E[{\mf[k]}]\group{\indset{\rval}}$ for every $k\in\set{1,\dots,\ell}$.
Also, $\E\group{0}\lo\E\group{\indset{\rval}}$ by Condition~\eqref{eq:null event lps}, so $0\ind\slo1\ind$ whence $1\ind\nrslo0\ind$ and hence, by Definition~\ref{def:Savage-null events}, $\ev$ is indeed no Savage-null event.


For the converse implication, consider any lexicographic probability system $\mf$ and assume that Condition~\eqref{eq:null event lps} does not hold.
Then there is some $\rval^*$ in $\ps$ such that $\mf[k]\group{\rval^*}=0$ for all $k$ in $\set{1,\dots,\ell}$, and therefore $\E[{\mf[k]}]\group{f\indset{\rval^*}}=0=\E[{\mf[k]}]\group{g\indset{\rval^*}}$ for all $f$ and $g$ in $\gambleson$ and $k$ in $\set{1,\dots,\ell}$, so $\E\group{f\indset{\rval^*}}=\E\group{g\indset{\rval^*}}$ for all $f$ and $g$ in $\gambleson$.
This implies that $f\indset{\rval^*}\rslo g\indset{\rval^*}$ for all $f$ and $g$ in $\gambleson$, so indeed there is a non-trivial Savage-null event $\set{\rval^*}$.
\end{proof}

\begin{proposition}\label{prop:lexicographic is strict weak order}
Consider any lexicographic probability system $\mf$ with $\ell$ layers.
Then $\slo$ is a \emph{strict weak order}, meaning that $\slo$ is irreflexive, and both $\slo$ and\/ $\ilo$ are transitive.
As a consequence, the relation $\nslo$ is transitive as well.
\end{proposition}

\begin{proof}
This is a consequence of Equations~\eqref{eq:order-mfs},~\eqref{eq:indif-lexic} and~\eqref{eq:non-dom-lexic}, taking into account that $\lo$ and $\rlo$ are transitive, and that $\lo$ is irreflexive.
\end{proof}

We now link the lexicographic orderings $\slo$ with the preference relation $\prel[\sodv]$ based on desirability, given by Equation~\eqref{eq:order-desirs}.
We begin with an auxiliary result:\footnote{Except for the second statement, most of the items in this propositions are well-known \citep[Section~1.4.1]{quaeghebeur2012:itip}; we include a simple proof for completeness.}

\begin{proposition}\label{prop:vector order}
Consider any lexicographic probability system $\mf$ with $\ell$ layers, and consider any coherent set of desirable gambles $\sodv$.
Then $\slo$ and $\prel[\sodv]$ are (strict) vector orders compatible with $\svo$: they are irreflexive, transitive and
\begin{enumerate}[label=\upshape(\roman*),leftmargin=*]
\item\label{it:prop:vector order:1} $f\slo g\iff f+h\slo g+h\iff\lambda f\slo\lambda g$;
\item\label{it:prop:vector order:2} if there are no non-trivial Savage-null events, then $f<g\then f\slo g$;
\item\label{it:prop:vector order:3} $f\prel[\sodv]g\iff f+h\prel[\sodv]g+h\iff\lambda f\prel[\sodv]\lambda g$;
\item\label{it:prop:vector order:4} $f<g\then f\prel[\sodv]g$,
\end{enumerate}
for all $f$, $g$ and $h$ in $\gambleson$ and $\lambda$ in $\posreals$.
\end{proposition}

\begin{proof}
It is clear from Proposition~\ref{prop:lexicographic is strict weak order} that $\slo$ is irreflexive and transitive.
To show that $\prel[\sodv]$ is irreflexive, infer from $f-f=0\notin\sodv$ [Axiom~\ref{coh sodv 1: 0 is not desirable}] that indeed $f\nprel[\sodv]f$ for all $f$ in $\gambleson$.
To show that $\prel[\sodv]$ is transitive, assume that $f\prel[\sodv]g$ and $g\prel[\sodv]h$.
Then $g-f\in\sodv$ and $h-g\in\sodv$, by Equation~\eqref{eq:order-desirs}, and hence $h-f=g-f+\group{h-g}\in\sodv$, by Axiom~\ref{coh sodv 4: sum is desirable}.
Using Equation~\eqref{eq:order-desirs} again, we find that then indeed $f\prel[\sodv]h$.
Let us now prove the remaining statements.
\begin{enumerate}[label=\upshape(\roman*),leftmargin=*]
\item This follows from the definition of $\prec$ and the linearity of the expectation operator. 
\item Assume that there are no non-trivial Savage-null events.
Use Proposition~\ref{prop:Savage-null} to infer that Condition~\eqref{eq:null event lps} holds.
Consider any $f$ in $\gambleson$ such that $0<f$.
Then $0\leq f$---so $0\leq\E[{\mf[k]}]\group{f}$ for every $k$ in $\set{1,\dots,\ell}$---and $0<f\group{\rval^*}$ for some $\rval^*$ in $\ps$.
Then $\mf[k]\group{\rval^*}>0$ for some $k$ in $\set{1,\dots,\ell}$ by Condition~\eqref{eq:null event lps}, so $0\lo\E\group{\indset{\rval^*}}$.
Use $f\group{\rval^*}\indset{\rval^*}\leq f$ to infer that then also $0\lo\E\group{f}$, whence indeed $0\slo f$.
\item The first equivalence follows immediately from Equation~\eqref{eq:order-desirs}, while the second is a consequence of the scaling axiom of coherent sets of desirable options. 
\item Assume that $f<g$.
Then $0<g-f$, whence $g-f\in\sodv$ by Axiom~\ref{coh sodv 2: positive vectors are desirable}.
Using Equation~\eqref{eq:order-desirs}, we find that then indeed $f\prel[\sodv]g$.\qedhere
\end{enumerate}
\end{proof}

Next we establish a link between lexicographic probability systems and preference relations associated with lexicographic sets of desirable gambles. 
We refer to papers by \citet[Section~2.1]{cozman2015} and \citet{seidenfeld1990} for other relevant discussion on the connection between lexicographic probabilities and partial preference relations. 
Our proof will make repeated use of the following separation theorem~\citep{holmes1975}, in the form stated by \citet[Appendix~E1]{walley1991}:

\begin{theorem}[Separating hyperplane theorem]\label{thm:separating}
Let\/ $\mathcal{W}_1$ and\/ $\mathcal{W}_2$ be two convex subsets of a finite-dimensional linear topological space $\mathcal{B}$.
If\/ $0\in\mathcal{W}_1\cap\mathcal{W}_2$ and\/ $\interior(\mathcal{W}_1)\cap\mathcal{W}_2=\emptyset$, then there is a non-zero continuous linear functional $\Lambda$ on $\mathcal{B}$ such that
\begin{equation*}
\Lambda(\vect[w])\geq0\text{ for all $\vect[w]$ in $\mathcal{W}_1$ and\/ }
\Lambda(\vect[w'])\leq0\text{ for all $\vect[w']$ in $\mathcal{W}_2$.} 
\end{equation*}
If\/ $\mathcal{W}_1$ and\/ $\mathcal{W}_2$ are finite, $\mathcal{W}_1$ non-empty, and $\sum_{i=1}^m\lambda_i\vect[w_i]-\sum_{k=1}^n\mu_k\vect[w_k']\neq0$ for all $m$ and $n$ in $\nats$, all $\lambda_1$, \dots, $\lambda_m$ in $\nonnegreals$ with $\lambda_i>0$ for at least one $i$ in $\set{1,\dots,m}$, all $\mu_1$, \dots, $\mu_n$ in $\nonnegreals$, all $\vect[w_1]$, \dots, $\vect[w_m]$ in $\mathcal{W}_1$, and all $\vect[w_1']$, \dots, $\vect[w_n']$ in $\mathcal{W}_2$, then there is a non-zero continuous linear functional $\Lambda$ on $\mathcal{B}$ such that
\begin{equation*}
\Lambda(\vect[w])>0\text{ for all $\vect[w]$ in $\mathcal{W}_1$ and\/ }
\Lambda(\vect[w'])\leq0\text{ for all $\vect[w']$ in $\mathcal{W}_2$.} 
\end{equation*}
\end{theorem}
\noindent 
Two clarifications here are (i) that we will apply the theorem to linear subsets of $\gambleson$, which is a linear topological space~\cite[Appendix~D]{walley1991} that is finite-dimensional because $\ps$ is finite, and (ii) that when the linear topological space is finite-dimensional, the assumption $\interior(\mathcal{W}_1)\neq\emptyset$ that \citet[Appendix~E1]{walley1991} mentions is not necessary for the separating hyperplane theorem to hold, as shown by \citet[Theorem~4B]{holmes1975}.

\begin{theorem}\label{thm:lexico sodvs}
Given a lexicographic probability system $\mf=(\mf[1],\dots,\mf[\ell])$ that has no non-trivial Savage-null events, the set of desirable gambles $\sodv[\mf]\coloneqq\cset{f\in\gambleson}{0\slo f}$ corresponding with the preference relation $\slo$, is an element of\/ $\cohlexisodvs$---a coherent and lexicographic set of desirable gambles.
Conversely, given a lexicographic set of desirable gambles $\sodv$ in $\cohlexisodvs$, its corresponding preference relation $\prel[\sodv]$ is a preference relation based on some lexicographic probability system $\mf=(\mf[1],\dots,\mf[\ell])$ that has no non-trivial Savage-null events.
\end{theorem}

\begin{proof}
We begin with the first statement.
We first show that $\sodv[\mf]$ is coherent.
For Axiom~\ref{coh sodv 1: 0 is not desirable}, infer from $0\nslo 0$ by the irreflexivity of $\slo$ [see Proposition~\ref{prop:lexicographic is strict weak order}] that indeed $0\notin\sodv[\mf]$.
For Axiom~\ref{coh sodv 2: positive vectors are desirable}, consider any $f$ in $\posgambles$.
Use Proposition~\ref{prop:vector order} to infer that $0\slo f$, whence indeed $f\in\sodv[\mf]$.
For Axiom~\ref{coh sodv 3: positive scaling is desirable}, consider any $f$ in $\sodv[\mf]$ and $\lambda$ in $\posreals$.
Then $0\slo f$, and hence $0\slo\lambda f$ using Proposition~\ref{prop:vector order}.
Then indeed $\lambda f\in\sodv[\mf]$.
For Axiom~\ref{coh sodv 4: sum is desirable}, consider any $f$ and $g$ in $\sodv[\mf]$, whence $0\slo f$ and $0\slo g$.
From $0\slo g$ infer that $f\slo f+g$ by Proposition~\ref{prop:vector order}, and using $0\slo f$, that $0\slo f+g$ by the transitivity of $\slo$ [see Proposition~\ref{prop:lexicographic is strict weak order}].
Then indeed $f+g\in\sodv[\mf]$.

So it only remains to show that $\Posi(\sodv[\mf]^c)=\sodv[\mf]^c$.
Consider any $f$ and $g$ in $\sodv[\mf]^c$ and any $\lambda_1$ and $\lambda_2$ in $\posreals$, then we must prove that $\lambda_1f+\lambda_2g\in\sodv[\mf]^c$.
Since by assumption $0\nslo f$ and $0\nslo g$, Equation~\eqref{eq:non-dom-lexic} guarantees that 
\begin{equation*}
\E\group{f}\rlo\E\group{0}=0
\text{ and }
\E\group{g}\rlo\E\group{0}.
\end{equation*}
By the linearity of the expectation operator,
\begin{equation*}
\E\group{\lambda_1 f+\lambda_2 g}\rlo\E\group{0}=0, 
\end{equation*}
whence
$0\nslo\lambda_1f+\lambda_2g$.
Then indeed $\lambda_1f+\lambda_2g\in\sodv[\mf]^c$.

For the second statement, we consider any $\sodv$ in $\cohlexisodvs$, and we construct a lexicographic probability system $\mf$ with no non-trivial Savage-null events and such that $\slo$ equals $\prel[\sodv]$.
Define the real functional $\Lambda_1$ on $\gambleson$ by letting $\Lambda_1(f)\coloneqq\sup\cset{\alpha\in\reals}{f-\alpha\in\sodv}$ for all $f$ in $\gambleson$.
Proposition~\ref{prop:lexicographic D induces linear prevision} guarantees that $\Lambda_1$ is a linear functional.
Its kernel $\ker\Lambda_1$ is an $n-1$-dimensional linear space, where $n$ is the finite dimension of the real vector space $\gambleson$---the cardinality of $\ps$.
Since both $\sodv^c$ and $\ker\Lambda_1$ are convex cones, so is their intersection $\sodv^c\cap\ker\Lambda_1$, and it contains $0$ because $0\in\sodv^c$ and $0\in\ker\Lambda_1$.
Using similar arguments, we see that $\sodv\cap\ker\Lambda_1$ is either a convex cone or empty.
When $\sodv\cap\ker\Lambda_1=\emptyset$, let $\ell\coloneq1$, and stop.
When $\sodv\cap\ker\Lambda_1\neq\emptyset$, it follows from Theorem~\ref{thm:separating} that 
there is some non-zero (continuous) linear functional $\Lambda_2$ on $\ker\Lambda_1$ such that
\begin{equation*}
\Lambda_2\group{f}\leq0\text{ for all $f$ in $\sodv^c\cap\ker\Lambda_1$ and }
\Lambda_2\group{f}\geq0\text{ for all $f$ in $\sodv\cap\ker\Lambda_1$}.
\end{equation*}
[Apply Theorem~\ref{thm:separating} with $\mathcal{B}=\ker\Lambda_1$, $\mathcal{W}_2=\sodv^c\cap\ker\Lambda_1$ and $\mathcal{W}_1=\closure\group{\sodv\cap\ker\Lambda_1}$ (the topological closure of $\sodv\cap\ker\Lambda_1$ in $\ker\Lambda_1$); then $\interior(\mathcal{W}_1)\cap\mathcal{W}_2=\emptyset$ by Lemma~\ref{lemma:int of cl is int}, and $0\in\mathcal{W}_1\cap\mathcal{W}_2$] $\ker\Lambda_2$ is a $n-2$-dimensional linear space.
Also, $\sodv\cap\ker\Lambda_2$ is either empty or a non-empty convex cone.
If it is empty, let $\ell\coloneqq2$; otherwise, we repeat the same procedure again: it follows from Theorem~\ref{thm:separating} that there is some non-zero (continuous) linear functional $\Lambda_3$ on $\ker\Lambda_2$ such that
\begin{equation*}
\Lambda_3\group{f}\leq0\text{ for all $f$ in $\sodv^c\cap\ker\Lambda_2$ and }
\Lambda_3\group{f}\geq0\text{ for all $f$ in $\sodv\cap\ker\Lambda_2$}.
\end{equation*}
[Apply Theorem~\ref{thm:separating} with $\mathcal{B}=\ker\Lambda_2$, $\mathcal{W}_2=\sodv^c\cap\ker\Lambda_2$ and $\mathcal{W}_1=\closure\group{\sodv\cap\ker\Lambda_2}$ (the topological closure of $\sodv\cap\ker\Lambda_2$ in $\ker\Lambda_2$); then $\interior(\mathcal{W}_1)\cap\mathcal{W}_2=\emptyset$ by Lemma~\ref{lemma:int of cl is int}, and $0\in\mathcal{W}_1\cap\mathcal{W}_2$] $\ker\Lambda_3$ is a $n-3$-dimensional linear space.
Also, $\sodv\cap\ker\Lambda_3$ is either empty or a non-empty convex cone.
If it is empty, let $\ell\coloneqq3$; if not, continue in the same vein.
This leads to successive linear functionals $\Lambda_k$ defined on the $n-k+1$-dimenional linear spaces $\ker\Lambda_{k-1}$ such that 
\begin{equation}\label{eq: separating hyperplane 1}
\Lambda_k\group{f}\leq0\text{ for all $f$ in $\sodv^c\cap\ker\Lambda_{k-1}$ and }
\Lambda_k\group{f}\geq0\text{ for all $f$ in $\sodv\cap\ker\Lambda_{k-1}$}.
\end{equation}
This sequence stops as soon as $\sodv\cap\ker\Lambda_k=\emptyset$, and we then let $\ell\coloneqq k$.
Because the finite dimensions of the successive $\ker\Lambda_k$ decrease with $1$ at each step, we are guaranteed to stop after at most $n$ repetitions: should $\sodv\cap\ker\Lambda_k\neq\emptyset$ for all $k\in\set{1,\dots,n-1}$ then $\ker\Lambda_n$ will be the $0$-dimensional linear space $\set{0}$, and then necessarily $\sodv\cap\ker\Lambda_n=\emptyset$.
For the last functional $\Lambda_\ell$, we have moreover that
\begin{equation}\label{eq:guaranteed:to:stop}
\Lambda_\ell\group{f}>0\text{ for all $f$ in $\sodv\cap\ker\Lambda_{\ell-1}$.}
\end{equation}
To see this, recall that by construction $\Lambda_\ell\group{f}\geq0$ for all $f$ in $\sodv\cap\ker\Lambda_{\ell-1}$, and that $\sodv\cap\ker\Lambda_\ell=\emptyset$.

In this fashion we obtain $\ell$ linear functionals $\Lambda_1$, \dots, $\Lambda_\ell$, each defined on the kernel of the previous functional---except for the domain $\gambleson$ of $\Lambda_1$.
We now show that we can turn the $\Lambda_2$, \dots, $\Lambda_\ell$ into expectation operators: positive and normalised linear functionals on the linear space $\gambleson$.
Indeed, consider their respective extensions $\Gamma_2$, \dots, $\Gamma_\ell$ to $\gambleson$ from Lemma~\ref{lem:from:kernel:to:allgambles} below, and let $\Gamma_1\coloneqq\Lambda_1$.
They satisfy $\Gamma_k(1)>0$ for all $k\in\set{1,\dots,\ell}$; see Proposition~\ref{prop:lexicographic D induces linear prevision} and Lemma~\ref{lem:from:kernel:to:allgambles}\ref{it:from:kernel:to:allgambles:nonzero}.
Now consider the real linear functionals on $\gambleson$ defined by $\E[1]\coloneqq\Gamma_1$, and $\E[k]\group{f}\coloneqq\Gamma_k\group{f}/\Gamma_k(1)$ for all $k$ in $\set{2,\dots,\ell}$ and $f$ in $\gambleson$.
It is obvious from Proposition~\ref{prop:lexicographic D induces linear prevision} and Lemma~\ref{lem:from:kernel:to:allgambles}\ref{it:from:kernel:to:allgambles:positive} that these linear functionals are normalised and positive, and therefore expectation operators on $\gambleson$.
Indeed each $\E[k]$ is the expectation operator associated with the mass function $\mf[k]$ defined by $\mf[k]\group{\rval}\coloneqq\E[k](\indset{\rval})$ for all $\rval$ in $\ps$.
In this way, $\mf\coloneqq(\mf[1],\dots,\mf[\ell])$ defines a lexicographic probability system.

We now prove that $\mf$ has no non-trivial Savage-null events, using Proposition~\ref{prop:Savage-null}.
Assume \emph{ex absurdo} that there is some $\rval^*$ in $\ps$ such that $\mf[k]\group{\rval^*}=\E[k]\group{\indset{\rval^*}}=0$ for all $k$ in $\set{1,\dots,\ell}$.
Then $\indset{\rval^*}\in\ker\Gamma_1=\ker\Lambda_1$ and $\indset{\rval^*}\in\Gamma_k$ for all $k$ in $\set{2,\dots,\ell}$.
Invoke Lemma~\ref{lem:from:kernel:to:allgambles}\ref{it:from:kernel:to:allgambles:kernels} to find that $\indset{\rval^*}\in\ker\Lambda_1\cap\ker\Gamma_2=\ker\Lambda_2$.
Repeated application of this same lemma eventually leads us to conclude that in $\indset{\rval^*}\in\ker\Lambda_{\ell-1}$ and $\indset{\rval^*}\in\ker\Lambda_\ell$.
Since also $\indset{\rval^*}\in\sodv$ and hence $\indset{\rval^*}\in\sodv\cap\ker\Lambda_{\ell-1}$ [Axiom~\ref{coh sodv 2: positive vectors are desirable}], Equation~\eqref{eq:guaranteed:to:stop} implies that $\Lambda_\ell\group{\indset{\rval^*}}>0$, a contradiction.

It now only remains to prove that $\prel[\sodv]$ is the lexicographic ordering with respect to \emph{this} lexicographic probability system, or in other words that
\begin{equation*}
f\in\sodv
\iff
0\lo(\E[1]\group{f},\dots,\E[\ell]\group{f})
\text{ for all $f$ in $\gambleson$.}
\end{equation*}

For necessity, assume that $f\in\sodv$.
Then $\E[1]\group{f}\geq0$ by the definition of $\Lambda_1$.
If $\E[1]\group{f}>0$, then we are done.
So assume that $\E[1]\group{f}=0$.
Then $f\in\ker\Lambda_1$ and $\Lambda_2\group{f}\geq0$ by Equation~\eqref{eq: separating hyperplane 1}.
Again, if $\Lambda_2\group{f}>0$, we can invoke Lemma~\ref{lem:from:kernel:to:allgambles}\ref{it:from:kernel:to:allgambles:samesign} to find that $\Gamma_2\group{f}>0$ and hence $\E[2]\group{f}>0$, and we are done.
So assume that $\Lambda_2\group{f}=0$.
Then $f\in\ker\Lambda_2$ and $\Lambda_3\group{f}\geq0$ by Equation~\eqref{eq: separating hyperplane 1}.
We can go on in this way, and we call $k$ the largest number for which $\E[j]\group{f}=0$ for all $j$ in $\set{1,\dots,k-1}$, or in other words, the smallest number for which $\E[k]\group{f}>0$.
Then $k\leq\ell$ by construction---see Equation~\eqref{eq:guaranteed:to:stop})---, whence indeed $0\lo(\E[1]\group{f},\dots,\E[\ell]\group{f})$.

For sufficiency, assume that $0\lo(\E[1]\group{f},\dots,\E[\ell]\group{f})$, meaning that there is some $k$ in $\set{1,\dots,\ell}$ for which $\E[j]\group{f}=0=\Gamma_j\group{f}$ for all $j$ in $\set{1,\dots,k-1}$ and $\E[k]\group{f}>0$, whence also $\Gamma_k\group{f}>0$.
So $f\in\ker\Gamma_j$ for all $j\in\set{1,\dots,k-1}$ and therefore repeated application of Lemma~\ref{lem:from:kernel:to:allgambles}\ref{it:from:kernel:to:allgambles:kernels} tells us that $f\in\ker\Lambda_j$ for all $j\in\set{1,\dots,k-1}$.
Since $\Gamma_k\group{f}>0$, we infer from Lemma~\ref{lem:from:kernel:to:allgambles}\ref{it:from:kernel:to:allgambles:samesign} that also $\Lambda_k\group{f}>0$, whence indeed $f\in\sodv$ by Equation~\eqref{eq: separating hyperplane 1}.
\end{proof}

\begin{lemma}\label{lemma:int of cl is int}
Consider any coherent set $\sodv$ of desirable gambles on a finite possibility space $\ps$, and consider any linear subspace $\Lambda\subseteq\gambleson$.
Then $\interior\group{\closure\group{\sodv\cap\Lambda}}\cap\sodv^c=\emptyset$, where $\interior$ is the topological interior and $\closure$ the topological closure.
\end{lemma}

\begin{proof}
We first prove $\interior\group{\closure\group{\sodv}}\cap\sodv^c=\emptyset$.
To show that, we will use the fact that $\sodv$, and therefore also $\closure\group{\sodv}$, is a convex set.
Since the interior of a convex set is always included in the relative interior $\relinterior$ of that convex set~\citep[see][Section~1.3]{Brondsted1983}, we find that $\interior\group{\closure\group{\sodv}}\subseteq\relinterior\group{\closure\group{\sodv}}$.
A well-known result~\cite[Theorem~3.4(d)]{Brondsted1983} states that $\relinterior\group{\closure\group{C}}=\relinterior\group{C}$ for any convex set $C$ in a finite-dimensional vector space, whence $\interior\group{\closure\group{\sodv}}\subseteq\relinterior\group{\sodv}$.
But $\relinterior\group{\sodv}$ is a subset of $\sodv$, so $\interior\group{\closure\group{\sodv}}\subseteq\sodv$, and hence indeed $\interior\group{\closure\group{\sodv}}\cap\sodv^c=\emptyset$.

Now consider $\sodv\cap\Lambda$, a subset of $\sodv$.
Since both $\closure$ and $\interior$ respect set inclusion, we find that $\interior\group{\closure\group{\sodv\cap\Lambda}}\subseteq\interior\group{\closure\group{\sodv}}\subseteq\sodv$, whence indeed $\interior\group{\closure\group{\sodv\cap\Lambda}}\cap\sodv^c=\emptyset$.
\end{proof}

\begin{lemma}\label{lem:from:kernel:to:allgambles}
Consider a non-zero real linear functional $\Lambda_1$ on the $n$-dimensional real vector space $\gambleson$, and a sequence of non-zero real linear functionals $\Lambda_k$ defined on the $n-k+1$-dimensional real vector space $\ker\Lambda_{k-1}$ for all $k$ in $\set{2,\dots,\ell}$, where $\ell\in\set{2,\dots,n}$.
Assume that all $\Lambda_k$ are positive in the sense that $\group{\forall f\in\nonneggambles\cap\domain\Lambda_k}\group{\Lambda_k(f)\geq0}$, for all $k\in\set{1,\dots,\ell}$. Then for each $k$ in $\set{2,\dots,\ell}$ the real linear functional $\Lambda_k$ on $\ker\Lambda_{k-1}$ can be extended to a real linear functional\/ $\Gamma_k$ on $\gambleson$ with the following properties:
\begin{enumerate}[label=\upshape(\roman*),leftmargin=*]
\item\label{it:from:kernel:to:allgambles:positive} For all $f$ in $\nonneggambles$: $\Gamma_k\group{f}\geq0$;
\item\label{it:from:kernel:to:allgambles:nonzero} $\Gamma_k(1)>0$;
\item\label{it:from:kernel:to:allgambles:kernels} $\ker\Gamma_k\cap\ker\Lambda_{k-1}=\ker\Lambda_k$;
\item\label{it:from:kernel:to:allgambles:samesign} For all $f$ in $\ker\Lambda_{k-1}$: $\Gamma_k\group{f}>0\iff\Lambda_k\group{f}>0$.
\end{enumerate}
\end{lemma}

\begin{proof}
Fix any $k$ in $\set{2,\dots,\ell}$. 
Since the real functional $\Lambda_k$ on the $n-k+1$-dimensional real vector space $\ker\Lambda_{k-1}$ is non-zero, there is some $h_k$ in $\ker\Lambda_{k-1}$ such that $\Lambda_k\group{h_k}>0$.
We will consider the quotient space $\qs$, a $k$-dimensional vector space whose elements $\ec[f]=f+\ker\Lambda_k$ are the affine subspaces through $f$, parallel to the subspace $\ker\Lambda_k$, for $f\in\gambleson$.
We first show that it follows from Theorem~\ref{thm:separating} that there is a non-zero linear functional $\tilde{\Gamma}_k$ on $\qs$ such that
\begin{multline}\label{eq:separating hyperplane 2}
\tilde{\Gamma}_k\group{u}\leq0\text{ for all $u$ in $\mathcal{W}_k^2\coloneqq\cset[\big]{\ec[{-\indset{\rval}}]}{\rval\in\ps_k}$, and }\\
\tilde{\Gamma}_k\group{u}>0\text{ for all $u$ in $\mathcal{W}_k^1\coloneqq\set{\ec[h_k]}\cup\cset[\big]{\ec[{\indset{\rval}}]}{\rval\in\ps_k}$,}
\end{multline}
where we let $\mathcal{X}_k\coloneqq\cset{\rval\in\ps}{\indset{\rval}\notin\ker\Lambda_k}\subseteq\ps$.
The set $\mathcal{X}_k$ is non-empty: since $\ker\Lambda_k$ is $n-k$-dimensional, at most $n-k$ of the linearly independent indicators $\indset{\rval}$, $\rval\in\ps$ may lie in $\ker\Lambda_k$, so $\abs{\ps_k}\geq k$.
To show that we can apply Theorem~\ref{thm:separating}, we prove that the condition for it is satisfied: $\sum_{i=1}^n\lambda_iw_i^1-\sum_{k=1}^m\mu_kw_k^2\neq0$ for all $m$ and $n$ in $\nats$, all $\lambda_1$, \dots, $\lambda_m$ in $\nonnegreals$ with $\lambda_i>0$ for at least one $i$ in $\set{1,\dots,m}$, all $\mu_1$, \dots, $\mu_n$ in $\nonnegreals$, all $w_1^1$, \dots, $w_n^1$ in $\mathcal{W}_k^1$, and all $w_1^2$, \dots, $w_m^2$ in $\mathcal{W}_k^2$.
Since $\mathcal{W}_k^1$ and $\mathcal{W}_k^2$ are finite, it is not difficult to see that it suffices to consider $\sum_{i=1}^n\lambda_iw_i^1=\lambda\ec[h_k]+\sum_{\rval\in\ps_k}\lambda_\rval\ec[{\indset{\rval}}]$ and $\sum_{j=1}^m\mu_jw_j^2=-\sum_{\rval\in\ps_k}\mu_\rval\ec[{\indset{\rval}}]$.
So assume \emph{ex absurdo} that $\lambda\ec[h_k]+\sum_{\rval\in\ps_k}\group{\lambda_\rval+\mu_\rval}\ec[{\indset{\rval}}]=0$, or equivalently, that $\lambda h_k+\sum_{\rval\in\ps_k}\group{\lambda_\rval+\mu_\rval}\indset{\rval}\in\ker\Lambda_k$ for some $\mu_\rval\geq0$, $\lambda_\rval\geq0$ and $\lambda\geq0$ for all $\rval$ in $\ps_k$, where $\lambda$ or at least one of $\cset{\lambda_\rval}{\rval\in\ps_k}$ are positive.
Let $\ps_k'\coloneqq\cset{\rval\in\ps_k}{\lambda_\rval+\mu_\rval>0}$ and $g\coloneqq\sum_{\rval\in\ps_k'}(\lambda_\rval+\mu_\rval)\indset{\rval}$, then we know that $\lambda h_k+g\in\ker\Lambda_k$.

There are now a number of possibilities.
The first is that $\lambda=0$, whence $\ps_k'\neq\emptyset$ and therefore $g\in\ker\Lambda_k\subseteq\dots\subseteq\ker\Lambda_1$.
This implies that $0=\Lambda_1(g)=\sum_{\rval\in\ps_k'}(\lambda_x+\mu_x)\Lambda_1\group{\indset{\rval}}$.
Since all $\indset{\rval}\svoi0$ and $\Lambda_1$ is positive, we find that $\indset{\rval}\in\ker\Lambda_1=\domain\Lambda_2$ for all $\rval$ in $\ps_k'$.
This in turn allows us to conclude that $0=\Lambda_2(g)=\sum_{\rval\in\ps_k'}(\lambda_x+\mu_x)\Lambda_2\group{\indset{\rval}}$.
Since all $\indset{\rval}\svoi0$ and $\Lambda_2$ is positive, we find that $\indset{\rval}\in\ker\Lambda_2=\domain\Lambda_3$ for all $\rval$ in $\ps_k'$.
We can go on in this way until we eventually conclude that $0=\Lambda_k(g)=\sum_{\rval\in\ps_k'}(\lambda_x+\mu_x)\Lambda_k\group{\indset{\rval}}$.
Since all $\indset{\rval}\svoi0$ and $\Lambda_k$ is positive, we find that $\indset{\rval}\in\ker\Lambda_k$ for all $\rval$ in $\ps_k'$, a contradiction.

The second possibility is that $\lambda>0$.
If now $\ps_k'=\emptyset$, we find that $\lambda h_k\in\ker\Lambda_k$, whence $\lambda\Lambda_k(h_k)=0$, a contradiction.
If $\ps_k'\neq\emptyset$, we find that $\lambda h_k+g\in\ker\Lambda_k\subseteq\dots\subseteq\ker\Lambda_1$. Since $h_k\in\ker\Lambda_{k-1}\subseteq\dots\subseteq\ker\Lambda_1$, this implies that $g\in\ker\Lambda_{k-1}\subseteq\dots\subseteq\ker\Lambda_1$ too.
This implies that $0=\Lambda_1(g)=\sum_{\rval\in\ps_k'}(\lambda_x+\mu_x)\Lambda_1\group{\indset{\rval}}$.
Since all $\indset{\rval}\svoi0$ and $\Lambda_1$ is positive, we find that $\indset{\rval}\in\ker\Lambda_1=\domain\Lambda_2$ for all $\rval$ in $\ps_k'$.
This in turn allows us to conclude that $0=\Lambda_2(g)=\sum_{\rval\in\ps_k'}(\lambda_x+\mu_x)\Lambda_2\group{\indset{\rval}}$.
Since all $\indset{\rval}\svoi0$ and $\Lambda_2$ is positive, we find that $\indset{\rval}\in\ker\Lambda_2=\domain\Lambda_3$ for all $\rval$ in $\ps_k'$.
We can go on in this way until we eventually conclude that $0=\Lambda_{k-1}(g)=\sum_{\rval\in\ps_k'}(\lambda_x+\mu_x)\Lambda_{k-1}\group{\indset{\rval}}$.
Since all $\indset{\rval}\svoi0$ and $\Lambda_{k-1}$ is positive, we find that $\indset{\rval}\in\ker\Lambda_{k-1}=\domain\Lambda_k$ for all $\rval$ in $\ps_k'$.
This now allows us to rewrite $\lambda h_k+g\in\ker\Lambda_k$ as $0=\Lambda_k(\lambda h_k+g)=\lambda\Lambda_k(h_k)+\sum_{\rval\in\ps_k'}(\lambda_\rval+\mu_\rval)\Lambda_k(\indset{\rval})$.
Since all $\indset{\rval}\svoi0$ and $\Lambda_k$ is positive, this implies that $\lambda\Lambda_k(h_k)\leq0$, a contradiction.
We conclude that, indeed, there is a non-zero linear functional $\tilde{\Gamma}_k$ on $\qs$ that satisfies Equation~\eqref{eq:separating hyperplane 2}.

We now define the new real linear functional $\Gamma_k$ on $\gambleson$ by letting
\begin{equation*}
\Gamma_k\group{f}\coloneqq\tilde{\Gamma}_k\group{\ec[f]}
\text{ for all $f$ in $\gambleson$.}
\end{equation*}
Observe that, since $f=\sum_{x\in\ps}f(x)\indset{x}$, this leads to
\begin{equation*}
\Gamma_k(f)
=\sum_{x\in\ps}f(x)\tilde{\Gamma}_k\group[\big]{\ec[\indset{x}]}
=\sum_{x\in\ps_k}f(x)\tilde{\Gamma}_k\group[\big]{\ec[\indset{x}]},
\end{equation*}
where the second equality follows from $\indset{x}\in\ker\Lambda_k$, and therefore $\ec[\indset{x}]=0$, for all $x\in\ps\setminus\ps_k$.
If we also take into account Equation~\eqref{eq:separating hyperplane 2}, this proves in particular that \ref{it:from:kernel:to:allgambles:positive} and~\ref{it:from:kernel:to:allgambles:nonzero} hold.

For the rest of the proof, consider any $f$ in $\ker\Lambda_{k-1}$ and $\lambda\coloneqq\Lambda_k\group{f}/\Lambda_k\group{h_k}$, a well-defined real number because $\Lambda_k(h_k)>0$.
Then $0=\Lambda_k\group{f}-\lambda\Lambda_k\group{h_k}=\Lambda_k\group{f-\lambda h_k}$, so $f-\lambda h_k\in\ker\Lambda_k$.
As a result, $\ec[f]=\ec[{\lambda h_k}]$ and therefore $\Gamma_k(f)=\tilde{\Gamma}_k(\ec[f])=\tilde{\Gamma}_k(\ec[{\lambda h_k}])=\lambda\tilde{\Gamma}_k(\ec[h_k])$.
Substituting back for $\lambda$, we get the equality:
\begin{equation*}
\Gamma_k(f)\Lambda_k\group{h_k}=\tilde{\Gamma}_k(\ec[h_k])\Lambda_k\group{f}.
\end{equation*}
Since both $\Lambda_k\group{h_k}>0$ and $\tilde{\Gamma}_k(\ec[h_k])>0$ [by Equation~\eqref{eq:separating hyperplane 2}], we see that $\Gamma_k(f)$ and $\Lambda_k\group{f}$ are either both zero, both (strictly) positive, or both (strictly) negative.
This proves~\ref{it:from:kernel:to:allgambles:kernels} and~\ref{it:from:kernel:to:allgambles:samesign}.
\end{proof}

We conclude that the sets of desirable options in $\cohlexisodvs$ are in a one-to-one correspondence with the lexicographic probability systems that have no non-trivial Savage-null events.
This is, of course, the reason why we have called the coherent sets of desirable options in $\cohlexisodvs\coloneqq\cset{\sodv\in\allcohsodvs}{\Posi\group{\sodv^c}=\sodv^c}$ \emph{lexicographic}.

Lexicographic probability systems can now also be related to specific types of choice functions, through Proposition~\ref{prop:CD coherent iff D lexicographic}: given a coherent set of desirable options $\sodv$, the most conservative coherent choice function $\cf[\sodv]$ whose binary choices are represented by $\sodv[\cf]=\sodv$ satisfies the convexity axiom~\ref{coh cf 5: convexity} if and only if $\sodv$ is a lexicographic set of desirable options.
We will call $\alllexicohcfs\coloneqq\cset{\cf[\sodv]}{\sodv\in\alllexicohsodvs}$ the set of \emph{lexicographic choice functions}.

Looking first at the most conservative coherent choice function that corresponds to $\sodv$ and then checking whether it is `convex', leads rather restrictively to lexicographic choice functions, and is only possible for lexicographic $\sodv$: convexity and choice based on Walley-Sen maximality are only compatible for lexicographic binary choice.
But suppose we turn things around somewhat, first restrict our attention to all `convex' coherent choice functions from the outset, and then look at the most conservative such choice function that makes the same binary choices as present in some given $\sodv$:
\begin{equation*}
\inf\cset{\cf\in\allcohcfs}
{\cf\text{ satisfies Axiom~\ref{coh cf 5: convexity} and }\sodv[\cf]=\sodv}.
\end{equation*}
We infer from Proposition~\ref{prop:inf of C5} that this infimum is still `convex' and coherent.
It will, of course, no longer be lexicographic, unless $\sodv$ is.
The following proposition tells us it still is an infimum of lexicographic choice functions. 

\begin{proposition}\label{prop:justification of lexicographic choice functions}
Consider an arbitrary coherent set of desirable options $\sodv$.
The most conservative of all coherent choice function $\cf$ that satisfies Axiom~\ref{coh cf 5: convexity} and $\sodv[\cf]=\sodv$ is the infimum of all lexicographic choice functions $\cf[\sodv']$ with $\sodv'$ in $\cohlexisodvs$ such that $\sodv\subseteq\sodv'$:
\begin{equation*}
\inf\cset{\cf\in\allcohcfs}
{\cf\text{ satisfies Axiom~\ref{coh cf 5: convexity} and }\sodv[\cf]=\sodv}
=
\inf\cset{\cf[\sodv']}
{\sodv'\in\cohlexisodvs\text{ and }\sodv\subseteq\sodv'}.
\end{equation*}
\end{proposition}

\begin{proof}
Denote the choice function on the left-hand side by $\cf[\mathrm{left}]$, and the one on the right-hand side by $\cf[\mathrm{right}]$.
Both are coherent, and so by Axiom~\ref{coh cf 4b: independence} completely characterised by the option sets from which $0$ is chosen.
Consider any $\os$ in $\cfdomo$, then we have to show that $0\in\cf[\mathrm{left}]\group{\set{0}\cup\os}\iff0\in\cf[\mathrm{right}]\group{\set{0}\cup\os}$.

For the direct implication, we assume that $0\in\cf[\mathrm{left}]\group{\set{0}\cup\os}$, meaning that there is some $\cf^*$ in $\allcohcfs$ that satisfies Axiom~\ref{coh cf 5: convexity}, $\sodv[\cf^*]=\sodv$ and $0\in\cf^*\group{\set{0}\cup\os}$.
We have to prove that there is some $\sodv^*$ in $\cohlexisodvs$ such that $\sodv\subseteq\sodv^*$ and $\sodv^*\cap\os=\emptyset$ [by Proposition~\ref{prop: CD in terms of intersection}], and we will do so by constructing a suitable lexicographic probability system, by a repeated application of an appropriate version of the separating hyperplane theorem [Theorem~\ref{thm:separating}], as in the proof of Theorem~\ref{thm:lexico sodvs}.

To prepare for this, we prove that $\Posi\group{\set{0}\cup\os}\cap\sodv=\emptyset$.
Indeed, assume \emph{ex absurdo} that $\Posi\group{\set{0}\cup\os}\cap\sodv\neq\emptyset$, so there is some $f\in\sodv$ such that $f\in\Posi\group{\set{0}\cup\os}$.
Then there is some $\lambda$ in $\posreals$ such that $g\coloneqq\lambda f\in\ch\group{\set{0}\cup\os}$.
Let $\os'\coloneqq\os\cup\set{g}$, so $\set{0}\cup\os'\subseteq\ch\group{\set{0}\cup\os}$, whence $0\in\cf^*\group{\set{0}\cup\os'}$ by Axiom~\ref{coh cf 5: convexity}, if we recall that $0\in\cf^*\group{\set{0}\cup\os}$.
But $f\in\sodv$ implies that $g\in\sodv$, and since $\sodv[\cf^*]=\sodv$, also that $g\in\sodv[\cf^*]$, or equivalently, $0\in\rf^*\group{\set{0,g}}$, by Proposition~\ref{prop: CD in terms of intersection}.
Version~\eqref{eq:senalpha} of Axiom~\ref{coh cf 3a: alpha} then guarantees that $0\in\rf^*\group{\set{0}\cup\os'}$, a contradiction.

It follows from this observation that we can apply Theorem~\ref{thm:separating} to show that there is some non-zero linear functional $\Lambda_1$ on $\gambles$ such that
\begin{equation}\label{eq:first:separation}
\Lambda_1\group{f}\leq0
\text{ for all $f$ in $\Posi\group{\set{0}\cup\os}$ and }
\Lambda_1\group{f}\geq0
\text{ for all $f$ in $\sodv$}.
\end{equation}
[Apply Theorem~\ref{thm:separating} with $\mathcal{B}=\gambleson$, $\mathcal{W}_2=\Posi\group{\set{0}\cup\os}$ and $\mathcal{W}_1=\sodv\cup\set{0}$, then $\interior\group{\mathcal{W}_1}\cap\mathcal{W}_2=\emptyset$ since $\interior\group{\mathcal{W}_1}\subseteq\sodv$, and $0\in\mathcal{W}_1\cap\mathcal{W}_2$.]
Its kernel $\ker\Lambda_1$ is an $n-1$-dimensional linear space, where $n$ is the dimension of $\gambleson$---the cardinality of $\ps$.
Since both $\sodv$ and $\ker\Lambda_1$ are convex cones, their intersection $\ker\Lambda_1\cap\sodv$ is either empty or a convex cone.
When $\ker\Lambda_1\cap\sodv=\emptyset$, we let $\ell\coloneqq1$, and stop.

When $\ker\Lambda_1\cap\sodv\neq\emptyset$, it follows from the same version of the separating hyperplane theorem that there is some non-zero linear functional $\Lambda_2$ on $\ker\Lambda_1$ such that
\begin{equation*}
\Lambda_2\group{f}\leq0
\text{ for all $f$ in $\ker\Lambda_1\cap\Posi\group{\set{0}\cup\os}$ and }
\Lambda_2\group{f}\geq0
\text{ for all $f$ in $\ker\Lambda_1\cap\sodv$}.
\end{equation*}
[Apply Theorem~\ref{thm:separating} with $\mathcal{B}=\ker\Lambda_1$, $\mathcal{W}_2=\Posi\group{\set{0}\cup\os}\cap\ker\Lambda_1$ and $\mathcal{W}_1=\group{\ker\Lambda_1\cap\sodv}\cup\set{0}$, then $\interior\group{\mathcal{W}_1}\cap\mathcal{W}_2=\emptyset$ since $\mathcal{W}_2\subseteq\Posi\group{\set{0}\cup\os}$ and $\interior\group{\mathcal{W}_1}\subseteq\sodv$, and $0\in\mathcal{W}_1\cap\mathcal{W}_2$.] $\ker\Lambda_2$ is a $n-2$-dimensional linear space.
As before, $\sodv\cap\ker\Lambda_2$ is either empty or a non-empty convex cone.
If it is empty, let $\ell\coloneqq2$; otherwise, repeat the same procedure over and over again, leading to successive non-zero linear functionals $\Lambda_k$ on $\ker\Lambda_{k-1}$ such that
\begin{equation}\label{eq:successive:separation}
\Lambda_k\group{f}\leq0
\text{ for all $f$ in $\ker\Lambda_{k-1}\cap\Posi\group{\set{0}\cup\os}$ and }
\Lambda_k\group{f}\geq0
\text{ for all $f$ in $\ker\Lambda_{k-1}\cap\sodv$},
\end{equation}
until eventually we get to the first $k$ such that $\sodv\cap\ker\Lambda_k=\emptyset$, and then let $\ell\coloneqq k$ and stop.
We are guaranteed to stop after at most $n$ repetitions, since $\ker\Lambda_n$ is the $0$-dimensional linear space $\set{0}$, for which $\sodv\cap\ker\Lambda_n=\emptyset$.
For the last functional $\Lambda_\ell$, we have that
\begin{equation}\label{eq:final:separation}
\Lambda_\ell\group{f}>0\text{ for all $f$ in $\sodv\cap\ker\Lambda_{\ell-1}$.}
\end{equation}
To see this, recall that by construction $\Lambda_\ell\group{f}\geq0$ for all $f$ in $\sodv\cap\ker\Lambda_{\ell-1}$, and that $\sodv\cap\ker\Lambda_\ell=\emptyset$.

In this fashion we obtain $\ell$ linear functionals $\Lambda_1$, \dots, $\Lambda_\ell$, each defined on the kernel of the previous functional---except for the domain $\gambleson$ of $\Lambda_1$.
We now show that we can turn the $\Lambda_1$, \dots, $\Lambda_\ell$ into expectation operators: positive and normalised linear functionals on the linear space $\gambleson$.
Indeed, consider their respective extensions $\Gamma_2$, \dots, $\Gamma_\ell$ to $\gambleson$ from Lemma~\ref{lem:from:kernel:to:allgambles}, and let $\Gamma_1\coloneqq\Lambda_1$.
They satisfy $\Gamma_k(1)>0$ for all $k$ in $\set{1,\dots,\ell}$; see Proposition~\ref{prop:lexicographic D induces linear prevision} and Lemma~\ref{lem:from:kernel:to:allgambles}\ref{it:from:kernel:to:allgambles:nonzero}.
Now consider the real linear functionals on $\gambleson$ defined by $\E[k]\group{f}\coloneqq\Gamma_k\group{f}/\Gamma_k(1)$ for all $k$ in $\set{1,\dots,\ell}$ and $f$ in $\gambleson$.
It is obvious from Lemma~\ref{lem:from:kernel:to:allgambles}\ref{it:from:kernel:to:allgambles:positive} that these linear functionals are normalised and positive, and therefore expectation operators on $\gambleson$.
Indeed each $\E[k]$ is the expectation operator associated with the mass function $\mf[k]$ defined by $\mf[k]\group{\rval}\coloneqq\E[k](\indset{\rval})$ for all $\rval$ in $\ps$.
In this way, $\mf\coloneqq(\mf[1],\dots,\mf[\ell])$ defines a lexicographic probability system.

We now prove that $\mf$ has no non-trivial Savage-null events, using Proposition~\ref{prop:Savage-null}.
Assume \emph{ex absurdo} that there is some $\rval^*$ in $\ps$ such that $\mf[k]\group{\rval^*}=\E[k]\group{\indset{\rval^*}}=0$ for all $k$ in $\set{1,\dots,\ell}$.
Then $\indset{\rval^*}\in\ker\Gamma_1=\ker\Lambda_1$ and $\indset{\rval^*}\in\Gamma_k$ for all $k$ in $\set{2,\dots,\ell}$.
Invoke Lemma~\ref{lem:from:kernel:to:allgambles}\ref{it:from:kernel:to:allgambles:kernels} to find that $\indset{\rval^*}\in\ker\Lambda_1\cap\ker\Gamma_2=\ker\Lambda_2$.
Repeated application of this same lemma eventually leads us to conclude that in $\indset{\rval^*}\in\ker\Lambda_{\ell-1}$ and $\indset{\rval^*}\in\ker\Lambda_\ell$.
Since also $\indset{\rval^*}\in\sodv$ and hence $\indset{\rval^*}\in\sodv\cap\ker\Lambda_{\ell-1}$ [Axiom~\ref{coh sodv 2: positive vectors are desirable}], Equation~\eqref{eq:final:separation} implies that $\Lambda_\ell\group{\indset{\rval^*}}>0$, a contradiction.

If we now let $\sodv^*\coloneqq\cset{f\in\gambleson}{0\lo(\E[1]\group{f},\dots,\E[\ell]\group{f})}$, then $\sodv^*\in\cohlexisodvs$ by Theorem~\ref{thm:lexico sodvs}.
If we can show that $\sodv\subseteq\sodv^*$ and $\os\cap\sodv^*=\emptyset$, we are done.
So first, consider any $f$ in $\sodv$.
Then $\Lambda_1(f)\geq0$ by Equation~\eqref{eq:first:separation}.
If $\Lambda_1(f)>0$ then also $\E[1](f)>0$ by Lemma~\ref{lem:from:kernel:to:allgambles}\ref{it:from:kernel:to:allgambles:nonzero}, and therefore $f\in\sodv^*$.
If $\Lambda_1(f)=0$ then $\Lambda_2(f)\geq0$ by Equation~\eqref{eq:successive:separation}.
If $\Lambda_2(f)>0$ then also $\E[2](f)>0$ by Lemma~\ref{lem:from:kernel:to:allgambles}\ref{it:from:kernel:to:allgambles:nonzero}\&\ref{it:from:kernel:to:allgambles:samesign}, and therefore $f\in\sodv^*$.
We can go on in this way until we get to the first $k$ for which $\Lambda_k(f)>0$, and therefore also $\E[k](f)>0$ by Lemma~\ref{lem:from:kernel:to:allgambles}\ref{it:from:kernel:to:allgambles:nonzero}\&\ref{it:from:kernel:to:allgambles:samesign}, whence therefore $f\in\sodv^*$.
We are guaranteed to find such a $k$ because we infer from Equation~\eqref{eq:final:separation} that $\Lambda_\ell(f)>0$.
This shows that indeed $\sodv\subseteq\sodv^*$.

Secondly, consider any $f$ in $\os$.
Then $\Lambda_1(f)\leq0$ by Equation~\eqref{eq:first:separation}.
If $\Lambda_1(f)<0$ then also $\E[1](f)<0$ by Lemma~\ref{lem:from:kernel:to:allgambles}\ref{it:from:kernel:to:allgambles:nonzero}, and therefore $f\notin\sodv^*$.
If $\Lambda_1(f)=0$ then $\Lambda_2(f)\leq0$ by Equation~\eqref{eq:successive:separation}.
If $\Lambda_2(f)<0$ then also $\E[2](f)<0$ by Lemma~\ref{lem:from:kernel:to:allgambles}\ref{it:from:kernel:to:allgambles:nonzero}\&\ref{it:from:kernel:to:allgambles:samesign}, and therefore $f\notin\sodv^*$.
If we go on in this way, only two things can happen: either there is a first $k$ for which $\Lambda_k(f)<0$, and therefore also $\E[k](f)<0$ by Lemma~\ref{lem:from:kernel:to:allgambles}\ref{it:from:kernel:to:allgambles:nonzero}\&\ref{it:from:kernel:to:allgambles:samesign}, whence therefore $f\notin\sodv^*$.
Or we find that $\Lambda_k(f)\leq0$, and therefore also $\E[k](f)\leq0$ by Lemma~\ref{lem:from:kernel:to:allgambles}\ref{it:from:kernel:to:allgambles:nonzero}\&\ref{it:from:kernel:to:allgambles:samesign}, for all $k\in\set{1,\dots,\ell}$, whence again $f\notin\sodv^*$.
This shows that indeed $A\cap\sodv^*=\emptyset$.

For the converse implication, assume that $0\in\cf[\mathrm{right}]\group{\set{0}\cup\os}$.
We must prove that there is some $\tilde{\cf}$ in $\allcohcfs$ that satisfies Axiom~\ref{coh cf 5: convexity}, $\sodv[\tilde{\cf}]=\sodv$ and $0\in\tilde{\cf}\group{\set{0}\cup\os}$.
We claim that $\tilde{\cf}\coloneqq\cf[\mathrm{right}]$ does the job.
Because we know by assumption that $0\in\cf[\mathrm{right}]\group{\set{0}\cup\os}$, and from Propositions~\ref{prop:CD coherent iff D lexicographic} and~\ref{prop:inf of C5} that $\cf[\mathrm{right}]$ is coherent and satisfies Axiom~\ref{coh cf 5: convexity}, it only remains to prove that $\sodv[{\cf[\mathrm{right}]}]=\sodv$.
To this end, consider any $f$ in $\gambleson$ and recall the following equivalences:
\begin{align*}
f\in\sodv[{\cf[\mathrm{right}]}]
&\iff
0\in\rf[\mathrm{right}]\group{\set{0,f}}
&&\text{[Equation~\eqref{eq:binary:behaviour}]}\\
&\iff
\group{\forall\sodv'\in\cohlexisodvs}
\group{\sodv\subseteq\sodv'\then0\in\rf[\sodv']\group{\set{0,f}}}
&&\text{[definition of $\inf$]}\\
&\iff
\group{\forall\sodv'\in\cohlexisodvs}
\group{\sodv\subseteq\sodv'\then f\in\sodv'}
&&\text{[Proposition~\ref{prop: CD in terms of intersection}]}\\
&\iff
f\in\sodv,
&&\text{[Proposition~\ref{prop: sodvs are dually atomic} and $\maxcohsodvs\subseteq\cohlexisodvs$]}
\end{align*}
which completes the proof.
\end{proof}

As a consequence of this result, we also have that, for any coherent set of desirable options $\sodv$, 
\begin{equation*}
\inf\cset{\cf\in\allcohcfs}
{\cf\text{ satisfies Axiom~\ref{coh cf 5: convexity} and }\sodv\subseteq\sodv[\cf]}
=
\inf\cset{\cf[\sodv']}
{\sodv'\in\cohlexisodvs\text{ and }\sodv\subseteq\sodv'}.
\end{equation*}

\section{Discussion and future research}\label{sec:conclusions}

One of the advantages of lexicographic probability systems is that they are more informative than single probability measures, and that they allow us to deal with some of the issues that arise when conditioning on sets of probability zero.
This is also the underlying idea behind some imprecise probability models, such as sets of desirable gambles.
In this paper, we have investigated the connection between the two models, by means of the more general theory of coherent choice functions.
We have shown that lexicographic probability systems correspond to the convexity axiom that was considered by Seidenfeld et al.
when considering choice functions on horse lotteries.
The study of this axiom has led to the consideration of what we have called lexicographic sets of desirable gambles.

In addition, we have also discussed the connection between our notion of coherent choice functions on abstract vectors, and the earlier notion for horse lotteries, developed mostly by \citet{seidenfeld2010}.
We have proved that by defining choice functions on arbitary vector spaces---something which also proves useful when studying the implications of an indifference assessment~\citep{vancamp2017}---we can include choice functions on horse lotteries as a particular case.
This allows us in particular to formulate our results for that framework.
Note, nevertheless, that there are some differences between \citeauthor{seidenfeld2010}'s \citeyearpar{seidenfeld2010} approach and ours, due to the rationality axioms considered and also to the fact that they deal with possibly infinite (but closed) sets of options, whereas our model assumes that choices are always made between finite sets of alternatives.
It would be interesting to investigate the extent to which our results can be generalised to infinite option sets.

One of the advantages of~\citeauthor{seidenfeld2010}'s \citeyearpar{seidenfeld2010} approach is that it leads to a representation theorem, in the sense that any coherent choice function can be obtained as the infimum of an arbitrary family of more informative convex coherent choice functions that essentially correspond to probability mass functions.
Based on the results and conclusions derived here, it seems natural to wonder if a similar result can be established in our framework.
Unfortunately, the answer to this question is negative: it turns out that in addition to convexity we need another axiom, which we have called \emph{weak Archimedeanity}.
With this extra axiom, at least for binary possibility spaces, it turns out a similar representation result can be proved: every such choice function is an infimum of its \emph{lexicographic} dominating choice functions, showing the importance of a study of lexicographic choice functions also from another angle of perspective.
The observation that we need an Archimedean axiom is in agreement with~\citeauthor{seidenfeld2010}'s \citeyearpar{seidenfeld2010} need of their \emph{Archimedean axiom}, which is----unlike our \emph{weak} Archimedeanity---difficult to join with desirability.
We intend to report on these results elsewhere.

On the other hand, we would also like to combine our results with the discussion by \citet{vancamp2017}, and investigate indifference and conditioning for the special case of lexicographic choice functions. 
In particular, this should allow us to link our work with \citeauthor{blume1991}'s \citeyearpar{blume1991} discussion of conditioning lexicographic probabilities. 
Finally, it may be interesting to generalise our results in Section~\ref{sec:lexicographic} to lexicographic probability systems defined on infinite spaces.

\end{document}